\documentclass{article}
\usepackage{mwe}
\usepackage{caption}
\usepackage{xfrac}
\usepackage{url}
\usepackage{wrapfig}

\usepackage{multirow}

\usepackage{booktabs}

\usepackage{tablefootnote}

\usepackage[numbers,compress,sort]{natbib}
\usepackage[final]{neurips_2024}

\usepackage{paperstyle}

\usepackage{blindtext}

\renewcommand{\paragraph}[1]{\textbf{#1}~~}

\usepackage{title_and_author}

\begin{document}
\maketitle

\begin{abstract}
Tuning scientific and probabilistic machine learning models -- for example, partial differential equations, Gaussian processes, or Bayesian neural networks -- often relies on evaluating functions of matrices whose size grows with the data set or the number of parameters.
While the state-of-the-art for \emph{evaluating} these quantities is almost always based on Lanczos and Arnoldi iterations, the present work is the first to explain how to \emph{differentiate} these workhorses of numerical linear algebra efficiently.
To get there, we derive previously unknown adjoint systems for Lanczos and Arnoldi iterations, implement them in JAX, and show that the resulting code can compete with Diffrax when it comes to differentiating PDEs, GPyTorch for selecting Gaussian process models and beats standard factorisation methods for calibrating Bayesian neural networks.
All this is achieved without any problem-specific code optimisation.
Find the code at {\texttt{https://github.com/pnkraemer/experiments-lanczos-adjoints}}
and install the library with {\texttt{pip install matfree}}.
\end{abstract}

\section{Introduction}
\label{section-introduction}

Automatic differentiation has dramatically altered the development of machine learning models by allowing us to forego laborious, application-dependent gradient derivations. 
The essence of this automation is to evaluate Jacobian-vector and vector-Jacobian products without ever instantiating the full Jacobian matrix, whose column count would match the number of parameters of the neural network. 
Nowadays, everyone can build algorithms around matrices of unprecedented sizes by exploiting this \emph{matrix-free} implementation. However, differentiable linear algebra for Jacobian-vector products and similar operations has remained largely unexplored to this day.
\emph{We introduce a new matrix-free method for automatically differentiating functions of matrices.} 
Our algorithm yields the exact gradients of the forward pass, all gradients are obtained with the same code, and said code runs in linear time- and memory-complexity.

For a parametrised matrix $A=A(\theta) \in \Rbb^{N \times N}$ and an analytic function $f: \Rbb \rightarrow \Rbb$, we call $f(A)$ a function of the matrix (different properties of $A$ imply different definitions of $f(A)$; one of them is applying $f$ to each eigenvalue of $A$ if $A$ is diagonalisable; see \citep{higham2008functions}).
However, we assume that $A$ is the Jacobian of a large neural network or a matrix of similar size and never materialise $f(A)$.
Instead, we only care about the values and gradients of the matrix-function-vector product
\begin{align}\label{equation-function-of-matrix}
(\theta, v) \mapsto f[A(\theta)] v
\end{align}
assuming that $A$ is only accessed via differentiable matrix-vector products.
\cref{table-example-applications} lists examples.
\begin{table}[t]
\caption{Some applications for functions of matrices. Log-determinants apply by combining $\log\det(A) = \trace{\log(A)}$ with stochastic trace estimation, which is why most vectors in this table are Rademacher samples. ``PDE'' / ``ODE'' = ``Partial/Ordinary differential equation''.}
\label{table-example-applications}
\begin{center}
\small
\begin{tabular}{l l l l l l}
\toprule
Application & Function $f$ & Matrix $A$ & Vector $v$ & Parameter $\theta$ &  \\
\midrule
PDEs \& flows \citep{gallopoulos1992efficient,hochbruck2010exponential,xiao2020generative,axelsson2014discrete}  & $e^\lambda$ & PDE discret. & PDE initial value & PDE  \\
Gaussian process \citep{gardner2018gpytorch,wenger2022preconditioning,immer2023stochastic} & $\log(\lambda)$ & Kernel matrix & $v \sim \text{Rademacher}$ & Kernel  \\
Invert. ResNets \citep{behrmann2019invertible,chen2019residual} & $\log(1+\lambda)$ &Jacobian matrix & $v \sim \text{Rademacher}$ & Network  \\
Gaussian sampler \citep{pleiss2018constant} & $\sqrt{\lambda}$ & Covariance matrix & $v \sim N(0, I)$ & Covariance \\
Neural ODE \citep{finlay2020train} & $\lambda^2$ & Jacobian matrix & $v \sim \text{Rademacher}$ & Network \\
\bottomrule
\end{tabular}
\end{center}
\end{table}

Evaluating \cref{equation-function-of-matrix} is crucial for building large machine learning models, e.g., Bayesian neural networks:
A common hyperparameter-calibration loss of a (Laplace-approximated) Bayesian neural network involves the log-determinant of the generalised Gauss--Newton matrix 
\citep{schraudolph2002fast}
\begin{align}\label{equation-gauss-newton-matrix}
A(\alpha) \coloneqq \sum_{(x_i,y_i) \in \text{data}} [D_\theta g](x_i)^\top [D_{g}^2 \rho](y_i, g(x_i)) [D_\theta g](x_i) + \alpha^2 I,
\end{align}
where $D_\theta g$ is the parameter-Jacobian of the neural network $g$, $D_g^2\rho$ is the Hessian of the loss function $\rho$ with respect to $g(x_i)$, and $\alpha$ is a to-be-tuned parameter.
The matrix $A(\alpha)$ in \cref{equation-gauss-newton-matrix} has as many rows and columns as the network has parameters, which makes traditional, cubic-complexity linear algebra routines for log-determinant estimation entirely unfeasible.
To compute this log-determinant, one chooses between either (i) simplifying the problem by pretending that the Hessian matrix is more structured than it actually is, e.g., diagonal \citep{daxberger2021laplace};
or (ii) approximating $\log\det(A)$ by combining stochastic trace estimation \citep{hutchinson1989stochastic}  
\begin{align}
\trace{A} = \Ebb\left[v^\top A v \right] \approx \frac{1}{L}\sum_{\ell=1}^L v_\ell^\top A v_\ell, \quad\text{for}\quad \Ebb\left[v v^\top\right] = I,
\end{align}
with a Lanczos iteration $A(\theta) \approx Q H Q^\top$ \citep{lanczos1950iteration}, to reduce the log-determinant to \citep{ubaru2017fast,chen2023krylov}
\begin{align}
\log\det(A) = \trace{\log{A}} \approx \frac{1}{L} \sum_{\ell=1}^L v_\ell^\top \log(A) v_\ell 
\approx
\frac{1}{L} \sum_{\ell=1}^L v_\ell^\top Q \log(H) Q^\top v_\ell.
\end{align}
The matrix $H$ in $A \approx Q H Q^\top$ has as many rows/columns as we are willing to evaluate matrix-vector products with $A$; thus, it is small enough to evaluate the matrix-logarithm $\log(H)$ in cubic complexity.

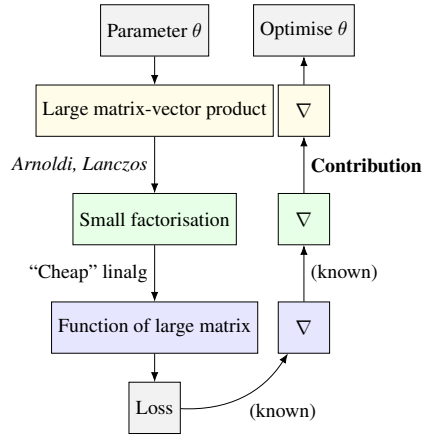
\begin{wrapfigure}[15]{r}{0.41\linewidth}
\vspace{-1\baselineskip}
\centering
\scalebox{0.75}{
\begin{tikzpicture}[]
\node (input)   [fill=yellow!10!white, draw=black, minimum size=0.91cm] { Large matrix-vector product};
\node (decomposition)   [below= of input,fill=green!10!white, draw=black, minimum size=0.91cm] { Small factorisation};
\node (matrixfunction)   [below=of decomposition,fill=blue!10!white, draw=black, minimum size=0.91cm] { Function of large matrix};
\node (loss)   [below= 0.5cm of matrixfunction,fill=gray!10!white, draw=black, minimum size=0.91cm] { Loss};
\node (grad-matrixfunction)   [right= 0.35cm of matrixfunction, fill=blue!10!white, draw=black, minimum size=0.91cm] { $\nabla$};
\node (grad-decomposition)   [above= of grad-matrixfunction, fill=green!10!white, draw=black, minimum size=0.91cm] { $\nabla$};
\node (grad-input)   [above= of grad-decomposition, fill=yellow!10!white, draw=black, minimum size=0.91cm] { $\nabla$};
\node (theta)   [fill=gray!10!white, draw=black, above= 0.5cm of input, minimum size=0.91cm] {Parameter $\theta$};
\node (goal) [fill=gray!10!white, draw=black, above= 0.5cm of grad-input, minimum size=0.91cm] {Optimise $\theta$};
\path[-Latex,draw] 
	(theta) edge node[above] {} (input)
	(input) edge node[left] {  \it Arnoldi, Lanczos } (decomposition)
	(decomposition) edge node[left] {``Cheap'' linalg } (matrixfunction)
	(matrixfunction) edge node[left] {  } (loss)
	(loss) edge[bend right] node[below right] { (known) } (grad-matrixfunction)
	(grad-matrixfunction) edge node[right] {  (known)  } (grad-decomposition)
	(grad-input) edge node[right] { } (goal);
\path[-Latex,draw]
	(grad-decomposition) edge node[right] { \textbf{{Contribution}} } (grad-input);
\end{tikzpicture}
}
\caption{Values (down) and gradients (up) of functions of large matrices.}
\label{figure-gradients-of-functions-of-large-matrices}
\end{wrapfigure}
\paragraph{Contributions}
This article explains how to differentiate not just log-determinants but {any Lanczos and Arnoldi iteration} so we can build loss functions for large models with such matrix-free algorithms (thereby completing the pipeline in \cref{figure-gradients-of-functions-of-large-matrices}).
This kind of functionality has been sorely missing from the toolbox of differentiable programming until now, even though the demand for functions of matrices is high in all of probabilistic and scientific machine learning \citep[e.g.][]{wang2019exact,wu2023large,konig2023efficient,davies2015effective,wenger2022preconditioning,dong2017scalable,pleiss2018constant,pleiss2020fast,gallopoulos1992efficient,wang2015krylov,xiao2020generative,chen2019residual,zhang2022fast,rezende2015variational,finlay2020train,hochbruck2010exponential,hoogeboom2020convolution,hochbruck1998exponential,gardner2018gpytorch,axelsson2014discrete,immer2021scalable,ritter2018scalable,behrmann2019invertible,immer2023stochastic}.

\section{Related work}
\label{section-example-applications}

Here, we focus on applications in machine learning and illustrate how prior work avoids differentiating matrix-free decomposition methods like the Lanczos and Arnoldi iterations. 
\citet{golub2009matrices} discuss applications outside machine learning.

\emph{Generative models}, e.g., normalising flows \citep{lipman2022flow,rezende2015variational}, rely on the change-of-variables formula, which involves the log-determinant of the Jacobian matrix of a neural network.
\citet{behrmann2019invertible} and \citet{chen2019residual} combine stochastic trace estimation with a Taylor-series expansion for the matrix logarithm.
\citet{ramesh2018backpropagation} use Chebyshev expansions instead of Taylor expansions.
That said, \citet{ubaru2017fast} demonstrate how both methods converge more slowly than the Lanczos iteration when combined with stochastic trace estimation.

\emph{Gaussian process model selection} requires values and gradients of log-probability density functions of Gaussian distributions (which involve log-determinants), where the covariance matrix $A(\theta)$ has as many rows and columns as there are data points \citep{williams2006gaussian}.
Recent work \citep{dong2017scalable,gardner2018gpytorch,wenger2022preconditioning,pleiss2018constant,wang2019exact,cutajar2016preconditioning} all uses some combination of stochastic trace estimation with the Lanczos iteration, and unanimously identifies gradients of log-determinants as (``$\diff$'' shall be an infinitesimal perturbation; see \cref{section-method})
\begin{align}
\label{equation-chain-rule-logdeterminant}
\mu \coloneqq \log\det(K(\theta)), 
\quad 
\diff \mu = \trace{K(\theta)^{-1} \diff K(\theta)}.
\end{align}
Another round of stochastic trace estimation then estimates $\diff \mu$ \citep{dong2017scalable,gardner2018gpytorch,cutajar2016preconditioning}.
In contrast, our contribution is more fundamental: not only do we derive the exact gradients of the forward pass, but our formulation also applies to, say, matrix exponentials, whereas \cref{equation-chain-rule-logdeterminant} only works for log-determinants.
\Cref{section-case-study-gaussian-processes} shows how our black-box gradients match state-of-the-art code for \Cref{equation-chain-rule-logdeterminant} \citep{gardner2018gpytorch}.

\emph{Laplace approximations and neural tangent kernels}
face the same problem of computing derivatives of log-determinants but with the generalised Gauss--Newton (GGN) matrix from \cref{equation-gauss-newton-matrix}.
In contrast to the Gaussian process literature, prior work on Laplace approximations prefers structured approximations of the GGN by considering subsets of network weights \citep{daxberger2021bayesian,kristiadi2020being,snoek2015scalable}, or algebraic approximations of the GGN via diagonal, KFAC, or low-rank factors \citep{denker1990transforming,ritter2018scalable,ritter2018online,maddox2020rethinking,sharma2021sketching,miani2022laplacian}.
All such approximations imply simple expressions for log-determinants, which are straightforward to differentiate automatically.
Unfortunately, these approximations discard valuable information about the correlation between weights, so a linear-algebra-based approach leads to superior likelihood calibration (\Cref{section-case-study-bayesian-neural-networks}).

\emph{Linear differential equations}, for instance $\dot y(t) = A y(t)$, $y(0) = y_0$ are solved by matrix exponentials, $y(t) = \exp(A t) y_0$.
By this relation, matrix exponentials have frequent applications not just for the simulation of differential equations \citep[e.g.][]{gallopoulos1992efficient,hochbruck1997krylov}, but also for the construction of exponential integrators \citep{hochbruck2010exponential,hochbruck1998exponential,zhang2022fast}, state-space models \citep{axelsson2014discrete,sarkka2019applied}, and in generative modelling \citep{xiao2020generative,hoogeboom2020convolution,zhang2022fast,rissanen2022generative}.
There are many ways of computing matrix exponentials \citep{moler1978nineteen,moler2003nineteen}, but only \citet{al2009computing} consider the problem of differentiating it and only in forward mode.
In contrast, differential equations have a rich history of adjoint methods \citep[e.g.][]{cao2003adjoint,diehl2011numerical} with high-performance open-source libraries \citep{rackauckas2017differentialequations,kidger2022neural,rackauckas2019diffeqflux,ma2021comparison}.
Still, the (now differentiable) Arnoldi iteration can compete with state-of-the-art solvers in JAX (\cref{section-case-study-pde}).

\section{Problem statement}
\label{section-problem-statement}

Recall $A=A(\theta) \in \Rbb^{N \times N}$ from \cref{section-introduction}.
The focus of this paper is on matrices that are too large to store in memory, like Jacobians of neural networks or discretised partial differential equations:
\begin{assumption}\label{assumption-matvecs-only}
$A(\theta)$ is only accessed via differentiable matrix-vector products $(\theta, v) \mapsto A(\theta) v$.
\end{assumption}
The \emph{de-facto} standard for linear algebra under \Cref{assumption-matvecs-only} are matrix-free algorithms  \citep[e.g.][Chapters 10 \& 11]{golub2013matrix}, like the conjugate gradient method for solving large sparse linear systems \citep{hestenes1952methods}.
But there is more to matrix-free linear algebra than conjugate gradient solvers:
\begin{wrapfigure}[8]{r}{0.4\textwidth}
\vspace{-0.75\baselineskip}
\centering
\scalebox{0.6}{
\begin{tikzpicture}

\node (A)   [draw=gray, fill=blue!20!white, rectangle, minimum width=3.5cm, minimum height=3.5cm] { \large $A(\theta)$ };

\node (Q1)   [draw=gray, fill=orange!20!white, rectangle, minimum width=0.8cm, minimum height=3.5cm, right=0.1cm of A] {\large  $Q$ };

\node (equal)   [right=0.05cm of Q1] {\large  $=$ };

\node (Q2)   [draw=gray, fill=orange!20!white, rectangle, minimum width=0.8cm, minimum height=3.5cm, right=0.05cm of equal] {\large  $Q$ };
\node (H)   [draw=gray, fill=green!20!white, rectangle, minimum width=0.8cm, minimum height=0.8cm, above right=-0.8cm and 0.1cm of Q2] {\large  $H$ };
\node (plus)   [right=0.85cm of Q2] {\large  $+$ };

\node (r)   [draw=gray, fill=red!20!white, rectangle, minimum width=0.2cm, minimum height=3.5cm, right=0.05cm of plus] {\large  $r$ };
\node (ek)   [draw=gray, fill=gray!20!white, rectangle, minimum width=0.8cm, minimum height=0.2cm, above right=-0.5cm and 0.1cm of r] {\large  $e_K$ };

\end{tikzpicture}
}
\caption{Lanczos/Arnoldi iteration.}
\label{figure-arnoldi-iteration}
\end{wrapfigure}
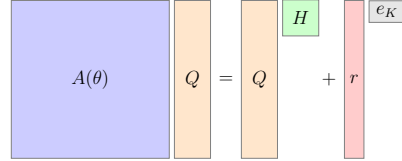
Matrix-free implementations of matrix decompositions usually revolve around variations of the Arnoldi iteration \citep{arnoldi1951principle}, which takes an initial vector $v \in \Rbb^N$ and a prescribed number of iterations $K \in \Nbb$ and produces a column-orthogonal $Q \in \Rbb^{N \times K}$, structured $H \in \Rbb^{K \times K}$, residual vector $r \in \Rbb^N$, and length $c \in \Rbb$ such that
\begin{align}
\label{equation-arnoldi-decomposition}
A Q = Q H + r (e_K)^\top, 
\quad
\text{and}
\quad
 Q e_1 = c v
\end{align}
hold (\Cref{figure-arnoldi-iteration}; $e_1, e_K \in \Rbb^{K}$ are the first and last unit vectors).
If $A$ is symmetric, $H$ is tridiagonal, and the Arnoldi iteration becomes the \emph{Lanczos iteration} \citep{lanczos1950iteration}.
Both iterations are popular for implementing matrix-function-vector products in a matrix-free fashion \citep{higham2008functions,golub2013matrix}, because
the decomposition in \Cref{equation-arnoldi-decomposition} implies $A \approx Q H Q^\top$, thus
\begin{align}
(\theta, v) \mapsto f(A(\theta)) v \approx Q f(H) Q^\top v =c^{-1} Q f(H) e_1.
\end{align}
The last step, $Q^\top v = c^{-1} e_1$, is due to the orthogonality of $Q$.
Since the number of matrix-vector products $K$ rarely exceeds a few hundreds or thousands, the following \Cref{assumption-small-matrix-function-differentiable} is mild:
\begin{assumption}
\label{assumption-small-matrix-function-differentiable}
The map $H \mapsto f(H) e_1$ is differentiable, and $Q$ fits into memory.
\end{assumption}
In summary, we evaluate functions of large matrices by firstly decomposing a large matrix into a product of small matrices (with Lanczos or Arnoldi) and, secondly, using conventional linear algebra to evaluate functions of small matrices.
Functions of small matrices can already be differentiated efficiently \citep{seeger2017auto,walter2010algorithmic,griewank2008evaluating}. 
\emph{This work contributes gradients of the Lanczos and Arnoldi iteration under \Cref{assumption-matvecs-only,assumption-small-matrix-function-differentiable}, and thereby makes matrix-free implementations of matrix decompositions and functions of large matrices (reverse-mode) differentiable.}
\begin{wrapfigure}[14]{r}{0.4\linewidth}
\centering
\includegraphics[width=\linewidth]{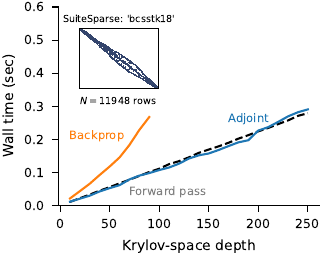}
\caption{%
Backpropagation vs our adjoint method on a sparse matrix \citep{kolodziej2019suitesparse,davis2011university,duff1989sparse}.
}
\label{figure-wall-times-per-vjp}
\end{wrapfigure}
Automatic differentiation, i.e. ``backpropagating through'' the matrix decomposition, is far too inefficient to be a viable option (\cref{figure-wall-times-per-vjp}; setup in \Cref{appendix-section-mode-vjp-wall-times}). Our approach via implicit differentiation or the adjoint method, respectively, leads to gradients that inherit the linear runtime and memory-complexity of the forward pass.

\paragraph{Limitations and future work}
The landscape of Lanczos- and Arnoldi-style matrix decompositions is vast, and some adjacent problems cannot be solved by this single article:
(i) {Forward-mode derivatives} would require a derivation separate from what comes next. Yet, since functions of matrices map many to few parameters (matrices to vectors), reverse-mode is superior to forward-mode anyway \citep[p. 153]{blondel2024elements}.
(ii) We only consider real-valued matrices (for their relevance to machine learning), even though the decompositions generalise to {complex arithmetic} with applications in physics \citep{groenenboom1990solving}.
(iii) We assume $Q$ fits into memory, which relates to combining Arnoldi/Lanczos with full reorthogonalisation \citep{paige1971computation,paige1972computational,paige1976error,borm2012numerical,golub2013matrix}. 
Relaxing this assumption requires gradients of {partial reorthogonalisation} (among other things), which we leave to future work.

\section{The method: Adjoints of the Lanczos and Arnoldi iterations}
\label{section-method}

Numerical algorithms are rarely differentiated automatically, and usually, some form of what is known as ``implicit differentiation'' \citep{blondel2024elements,blondel2022efficient} applies.
The same is true for the Lanczos and Arnoldi iterations.
However, and perhaps surprisingly, we differentiate the iterations like a dynamical system using the ``adjoint method'' \citep{cao2003adjoint,betancourt2020discrete,tran2024geometric}, a variation of implicit differentiation that uses Lagrange multipliers \citep{blondel2024elements}, and not like a linear algebra routine \citep{seeger2017auto,walter2010algorithmic,roberts2020qr}.
To clarify this distinction, we briefly review implicit differentiation before the core contributions of this work in \Cref{section-adjoint-systems,section-matrix-free-implementation}.

\paragraph{Notation}
Let $\diff x$ be an infinitesimal perturbation of some $x$. $D$ is the Jacobian operator, and $\langle \cdot, \cdot \rangle$ the Euclidean inner product between two equally-sized inputs.
For a loss $\rho \in \Rbb$ that depends on some $x$, the linearisation $\diff \rho = D_x \rho\, \diff x$ and the gradient identity
$\diff \rho = \langle \nabla_x \rho, \diff x \rangle$ will be important \citep{minka2000old,giles2008extended}.

\paragraph{Implicit differentiation}
Let $\abf: \theta \mapsto x$ be a numerical algorithm that computes some $x$ from some $\theta$.
Assume that the input and output of $\abf(\cdot)$ satisfy the constraint $\cbf(\theta, \abf(\theta)) = 0$.
For instance, if $\abf(\cdot)$ solves $Ax=b$, the constraint is $\cbf(A, b; x) = Ax - b$ with $\theta \coloneqq \{A, b\}$.
We can use $\cbf(\cdot)$ in combination with the chain rule to find the derivatives of $\abf(\cdot)$, { ($\cbf = 0$ implies $\diff \cbf = 0$)}
\begin{align}\label{equation-implicit-constraint-linearised}
0 
= \diff \cbf(\theta, x) = D_x \cbf(\theta, x) \diff x + D_\theta \cbf(\theta, x) \diff \theta.
\end{align}
In other words, we ``linearise'' the constraint $\cbf(\theta, x) = 0$.
The adjoint method \citep{cea1986conception} proceeds by ``transposing'' this linearisation as follows. 
Let $\rho$ be a loss that depends on $y$ with gradient $\nabla_y \rho$ and recall the gradient identity from the ``Notation'' paragraph above.
Then, for all Lagrange multipliers $\lambda$ with the same shape as the outputs of  $\cbf(\cdot)$, we know that since $\cbf = 0$ implies $\diff \cbf = 0$,
\begin{align}\label{equation-adjoint-method-explained}
\diff \rho 
= \langle \nabla_x \rho, \diff x \rangle
= \langle \nabla_x \rho, \diff x \rangle + \langle \lambda, \diff \cbf \rangle 
= \langle \nabla_x \rho + (D_x \cbf)^\top \lambda, \diff x \rangle + \langle (D_\theta \cbf)^\top \lambda, \diff \theta \rangle
\end{align}
must hold.
By matching \Cref{equation-adjoint-method-explained} to $\diff \rho = \langle \nabla_\theta \rho,  \diff \theta\rangle$ (this time, regarding $\rho$ as a function of $\theta$, not of $x$; recall the ``Notation'' paragraph), we conclude that if $\lambda$ solves the adjoint system
\begin{align} \label{equation-adjoint-system}
\nabla_x \rho + (D_x \cbf)^\top \lambda = 0,
\end{align}
then $\nabla_\theta \rho \coloneqq (D_\theta \cbf)^\top \lambda$ must be the gradient of $\rho$ with respect to input $\theta$.
This is the adjoint method \citep[Section 10.4]{blondel2024elements}.
In automatic differentiation frameworks like JAX \citep{jax2018github}, this gradient implements a vector-Jacobian product with the Jacobian of $\abf(\cdot)$ -- implicitly via the Lagrange multiplier $\lambda$, without differentiating ``through'' $\abf(\cdot)$ explicitly.
In comparison to approaches that explicitly target vector-Jacobian products with implicit differentiation \citep[like][Proposition 10.1]{blondel2024elements}, the adjoint method shines when applied to highly structured, non-vector-valued constraints, 
such as dynamical systems or the Lanczos and Arnoldi iterations.
The reason is that the adjoint method does not change if $\cbf(\cdot)$ becomes matrix- or function-space-valued, as long as we can define inner products and adjoint operators, whereas other approaches (like what \citet{blondel2022efficient} use for numerical optimisers) would become increasingly laborious in these cases.
In summary, to reverse-mode differentiate a numerical algorithm with the adjoint method, we need four steps: (i) find a constraint, (ii) linearise it, (iii) introduce Lagrange multipliers, and (iv) solve the resulting adjoint system.
Carrying out those four steps for the Lanczos and Arnoldi iterations is the main contribution of the paper: \Cref{section-adjoint-systems} states both adjoint systems and \Cref{section-matrix-free-implementation} covers a matrix-free implementation.

\subsection{Adjoint system of the Arnoldi and Lanczos iterations}
\label{section-adjoint-systems}
Let $e_j$ be the $j$th unit vector. Denote by ``$\circ$'' the element-wise matrix product, and define the matrices
\begin{align}
I_\leq \coloneqq [\delta_{i \leq j}]_{i,j=1}^K, 
\quad
I_< \coloneqq [\delta_{i<j}]_{i,j=1}^K,
\quad
I_\ll \coloneqq [\delta_{i+1<j}]_{i,j=1}^K,
\end{align}
so that for example, $I_\leq \circ A$ extracts the lower triangular matrix of $A$ (including the diagonal), and $I_\ll \circ A = 0$ enforces Hessenberg form \citep{golub2013matrix}.
The following two theorems do not require \Cref{assumption-matvecs-only,assumption-small-matrix-function-differentiable}, which are only relevant for analysing the computational complexities. 
\begin{theorem}[Adjoint system of the Arnoldi iteration]
\label{theorem-adjoint-system-of-arnoldi}
Let $K \in \Nbb$, $v \in \Rbb$, and $A \in \Rbb^{N \times N}$, and a loss $\rho(\cdot) \in \Rbb$ be given. 
If $Q \in \Rbb^{N \times K}$, $H \in \Rbb^{K \times K}$, $r \in \Rbb^N$, and $c \in \Rbb$ solve the \emph{forward constraint}
\begin{align}
\label{equation-arnoldi-constraint}
AQ = QH + r (e_K)^\top, 
\quad
Q e_1 = vc, 
\quad
I_\leq \circ [Q^\top Q]=I, 
\quad
I_\ll \circ H = 0,
\quad
 Q^\top r = 0,
\end{align}
and if $\lambda \in \Rbb^N$, $\Lambda \in \Rbb^{N \times K}$, $\gamma \in \Rbb^{K}$, $\Gamma \in \Rbb^{K \times K}$, and $\Sigma \in \Rbb^{K \times K}$ satisfy the \emph{adjoint system}
\begin{subequations}\label{equation-adjoint-system-arnoldi}
\begin{align}
0 
&=
\nabla_Q \rho + A^\top \Lambda - \Lambda H^\top + \lambda (e_1)^\top + Q(I_\leq \circ \Gamma) + Q(I_\leq \circ \Gamma)^\top + r \gamma^\top 
\label{equation-adjoint-constraint-dynamics}
\\
0
&=
\nabla_H \rho - Q^\top \Lambda + I_\ll \circ \Sigma
\label{equation-adjoint-constraint-projection}
\\
0
&=
\nabla_r \rho - \Lambda e_K + Q \gamma 
\label{equation-adjoint-constraint-on-final-multiplier}
\\
0
&=
\nabla_c \rho - v^\top \lambda,
\label{equation-adjoint-constraint-scalar}
\end{align}
\end{subequations}
then the gradients of $\rho$ with respect to $A$ and $v$ are 
\begin{align}\label{equation-gradients-of-arnoldi}
\nabla_A \rho \coloneqq \Lambda Q^\top, \quad \nabla_v \rho \coloneqq \lambda c.
\end{align}
\end{theorem}
\begin{proof}[Sketch of the proof]
To derive the statement, start with \Cref{equation-arnoldi-constraint} as $\cbf(\cdot)$.
Apply the chain- and product rules liberally to get $\diff \cbf(\cdot)$.
Introduce Lagrange multipliers $\lambda$, $\Lambda$, $\gamma$, $\Gamma$, and $\Sigma$ like in the previous section, by adding Lagrange-multiplied constraints to
\begin{align}
\diff \rho  =
\langle \nabla_Q \rho, \diff Q \rangle 
+\langle \nabla_H \rho, \diff H \rangle 
+\langle \nabla_r \rho, \diff r \rangle 
+\langle \nabla_c \rho, \diff c \rangle,
\end{align}
and rearrange the terms to see that \Cref{equation-adjoint-system-arnoldi} implies \Cref{equation-gradients-of-arnoldi}.
Details are in \Cref{appendix-section-proof-of-theorem-adjoint-system-of-arnoldi}.
\end{proof}
\begin{theorem}[Adjoint system of the Lanczos iteration]
\label{theorem-adjoint-system-of-lanczos}
Let a symmetric $A \in \Rbb^{N \times N}$, as well as $v \in \Rbb^N$, $K \in \Nbb$, and a loss $\rho$ be known.
In the following equations, set $b_0\coloneqq 1 \in \Rbb$, $x_0 \coloneqq 0 \in \Rbb^n, \lambda_{K+1}\coloneqq 0$, $\mu_0 \coloneqq 0$, $\nu_0 \coloneqq 0$ to simplify the expressions.
If $x_1, ..., x_{K+1} \in \Rbb^{N}$, and $a_1, ..., a_k, b_1, ..., b_k \in \Rbb^{K}$, satisfy the \emph{forward constraint}
\begin{subequations}
\begin{align}
x_1 - v / (v^\top v)  &= 0,
\\
-b_{k-1}x_{k-1} + (A-a_kI) x_k  -b_k x_{k+1} &= 0,
&
k&=1, ..., K 
\\
x_{k+1}^\top x_{k+1} - 1&= 0, 
&
k&=1, ..., K,
\\
x_{k-1}^\top x_k &= 0, 
&
k&=2, ..., K+1 
\end{align}
\end{subequations}
and if $\lambda_0, ..., \lambda_K \in \Rbb^N$, $\mu_1, ..., \mu_K, \nu_1, ..., \nu_K \in \Rbb$ satisfy the \emph{adjoint system}
\begin{subequations}
\begin{align}
0	&= 
		- \lambda_{K} b_K 
		+ (
			\nabla_{x_{K+1}} \rho 
			+  \mu_K x_{K+1} 
			+ \nu_{K} x_K
		),  
\\
0 
	&= 
		- b_k \lambda_{k+1} 
		+ (A^\top - a_k I)\lambda_{k} 
		- b_{k-1} \lambda_{k-1} 
		+ (
			\nabla_{x_k}\rho
			+ \mu_{k-1} x_{k} 
			+ \nu_k x_{k+1}
			+ \nu_{k-1} x_{k-1}
		),  
	\\
0	&= 
		\nabla_{a_k} \rho 
		- \lambda_{k}^\top x_k,
		\\
0	&= 
		\nabla_{b_k} \rho 
		- \lambda_{k+1}^\top x_k 
		- \lambda_{k}^\top x_{k+1}.
\end{align}
\end{subequations}
where all expressions involving $k$ hold for all $k=K, ..., 1$,
then 
\begin{align}
\nabla_v \rho \coloneqq \frac{\lambda_0^\top x_1}{v^\top v} x_1- \lambda_0, 
\quad 
\nabla_A \rho \coloneqq \sum_{k=1}^K \lambda_{k} x_k^\top
\end{align}
are the gradients of $\rho$ with respect to $v$ and $A$.
\end{theorem}
\begin{proof}[Sketch of the proof]
This theorem is proven similarly to that of \Cref{theorem-adjoint-system-of-arnoldi}, but instead of a few equations involving matrices, we have many equations involving scalars because for symmetric matrices, $H$ must be tridiagonal \citep{golub2013matrix}, and we expand $AQ = QH + r (e_{K+1})^\top$ column-wise. The coefficients $a_k$ and $b_k$ are the tridiagonal elements in $H$. We rename $q_k$ from Arnoldi to $x_k$ for Lanczos to make it easier to distinguish the two different sets of constraints.
Details: \Cref{appendix-section-proof-of-theorem-adjoint-system-of-lanczos}.
\end{proof}

\subsection{Matrix-free implementation}
\label{section-matrix-free-implementation}

\paragraph{Solving the adjoint systems}
To compute $\nabla_A \rho$ and $\nabla_v \rho$, we need to solve the adjoint systems.
When comparing the forward constraints to the adjoint systems, similarities emerge: for instance, the adjoint system of the Arnoldi iteration follows the same $A^{(\top)} H - \Lambda H^{(\top)} + \text{rest} = 0$ structure as the forward constraint.
This structure suggests deriving a recursion for the backward pass that mirrors that of the forward pass. \Cref{appendix-section-solving-the-adjoint-system-in-detail} contains this derivation and contrasts the resulting algorithm with that of the forward pass.
The main observation is that the complexity of the adjoint passes for Lanczos and Arnoldi mirrors that of the forward passes. Gradients can be implemented purely with matrix-vector products, which is helpful because it makes our custom backward pass as matrix-free as backpropagation ``through'' the forward pass would be.
This matrix-free implementation in combination with the efficient recursions in \Cref{theorem-adjoint-system-of-arnoldi,theorem-adjoint-system-of-lanczos}  explains the significant performance gains of our method compared to naive backpropagation, observed in \Cref{figure-wall-times-per-vjp}.

\Cref{theorem-adjoint-system-of-arnoldi,theorem-adjoint-system-of-lanczos}'s expressions for $\nabla_A \rho$ are not directly applicable when we only have matrix-vector products with $A$.
Fortunately, parameter-gradients emerge from matrix-gradients:
\begin{corollary}[Parameter gradients] 
\label{corollary-parameter-gradients-arnoldi}
Under \Cref{assumption-matvecs-only} and the assumptions of \cref{theorem-adjoint-system-of-arnoldi}, and if $A$ is parametrised by some $\theta$, the gradients of $\rho$ with respect to $\theta$ are 
\begin{align}
\nabla_\theta \rho = \sum_{k=1}^K  \nabla \left[\theta \mapsto (e_k)^\top Q^\top A(\theta)^\top \Lambda e_k \right],
\end{align}
which can be assembled online during the backward pass.
For the Lanczos iteration, we assume the conditions of \Cref{theorem-adjoint-system-of-lanczos} instead of \Cref{theorem-adjoint-system-of-arnoldi}, replace $Qe_k$ and $\Lambda e_k$ with $x_k$ and $\lambda_k$, let the sum run from $k=0$ to $k=K$, and the rest of this statement remains true.
\end{corollary}
\begin{proof}[Sketch of the proof]
The proof of this identity combines the expression(s) for $\nabla_A \rho$ from \Cref{theorem-adjoint-system-of-arnoldi,theorem-adjoint-system-of-lanczos} with $\diff A = D_\theta A \diff \theta$.
The derivations are lengthy and therefore relegated to \Cref{appendix-section-proof-of-corollary-parameter-gradients-arnoldi}.
\end{proof}
\begin{wrapfigure}[13]{r}{0.4\linewidth}
\vspace{-1.4\baselineskip}
\begin{center}
\captionof{table}{Accuracy loss when differentiating the Arnoldi iteration on a Hilbert matrix in double precision ($\phi:$ decompose with a full-rank Arnoldi iteration, then reconstruct the original matrix; measure $\|\partial \phi - I\|$; details in \Cref{appendix-section-accuracy-loss-hilbert}).}
\label{table-accuracy-loss-hilbert}
\begin{tabular}{l c }
\toprule
& Loss of accuracy\\
\midrule
Adjoint w/o proj. &$5.83 \cdot 10^{-3}$ \\
Adjoint w/ proj. & \cellcolor[gray]{0.9}$\mathbf{1.17 \cdot 10^{-10}}$ \\
Backprop. & \cellcolor[gray]{0.9}$\mathbf{1.17 \cdot 10^{-10}}$\\
\bottomrule
\end{tabular}
\end{center}
\end{wrapfigure}
\paragraph{Reorthogonalisation}
It is well known that the Lanczos and Arnoldi iterations suffer from a loss of orthogonality and that reorthogonalisation of the columns in $Q$ is often necessary \citep{paige1971computation,paige1972computational,paige1976error,borm2012numerical}. 
Reorthogonalisation does not affect the forward constraints, so the adjoint systems remain the same with and without reorthogonalisation.
But adjoint systems also suffer from a loss of orthogonality:
The equivalent of orthogonality for the adjoint system is the projection constraint in \cref{equation-adjoint-constraint-projection}, which constrains the Lagrange multipliers $\Lambda$ to a hyperplane defined by $Q$ and other known quantities.
The constraint can -- and should (\cref{table-accuracy-loss-hilbert}) -- be used whenever the forward pass requires reorthogonalisation.%
\footnote{The adjoint system of the Lanczos iteration does not admit this projection constraint, but we can implement re-orthogonalised Lanczos via calling the Arnoldi code. This induces only minimal overhead because fully reorthogonalised Lanczos code has roughly the same complexity as Arnoldi code.}
In the case studies below, we always use full reorthogonalisation on the forward and adjoint pass, also for the Arnoldi iteration \citep[Table 7.1]{borm2012numerical}, even though this is slightly less common than for the Lanczos iteration.

\paragraph{Summary (before the case studies)}
The main takeaway from \Cref{section-adjoint-systems,section-matrix-free-implementation} is that now, we do not only have closed-form expressions for the gradients of Arnoldi and Lanczos iterations (\Cref{theorem-adjoint-system-of-arnoldi,theorem-adjoint-system-of-lanczos}), but that we can compute them in the same complexity as the forward pass, in a numerically stable way, and evaluate parameter-gradients in linear time- and space-complexity (\Cref{corollary-parameter-gradients-arnoldi}).
While some of the derivations are somewhat technical, the overall approach follows the general template for the adjoint method relatively closely.
The resulting algorithm beats backpropagation ``through'' the iterations by a margin in terms of speed (\Cref{figure-wall-times-per-vjp}) and enjoys the same stability gains from reorthogonalisation as the forward pass (\Cref{table-accuracy-loss-hilbert}).
Our open-source implementation of reverse-mode differentiable Lanczos and Arnoldi iterations can be installed via ``\texttt{pip install matfree}''.
Next, we put this code to the test on three challenging machine-learning problems centred around functions of matrices to see how it fares against state-of-the-art differentiable implementations of exact Gaussian processes (\Cref{section-case-study-gaussian-processes}), differential equation solvers (\Cref{section-case-study-pde}), and Bayesian neural networks (\Cref{section-case-study-bayesian-neural-networks}).

\section{Case study: Exact Gaussian processes}
\label{section-case-study-gaussian-processes}

Model selection for Gaussian processes has arguably been the strongest proponent of the Lanczos iteration and similar matrix-free algorithms in recent years \citep{gardner2018gpytorch,dong2017scalable,wenger2022preconditioning,wu2023large,wang2019exact,pleiss2018constant,wilson2015kernel}, and most of these efforts have been bundled up in the GPyTorch library \citep{gardner2018gpytorch}.
For example, GPyTorch defaults to choosing a Lanczos iteration over a Cholesky decomposition as soon as the dataset exceeds 800 data points.\footnote{Parameter \texttt{max\_cholesky\_size}: \texttt{https://docs.gpytorch.ai/en/stable/settings.html}.}
Calibrating hyperparameters of Gaussian process models involves optimising log-marginal-likelihoods of the regression targets, which requires computing $x^\top A^{-1} x$ and $\log\det(A)$ for a covariance matrix $A$ with as many rows and columns as there are data points. 
Recent works \citep{gardner2018gpytorch,dong2017scalable,wenger2022preconditioning,wu2023large,wang2019exact,pleiss2018constant,wilson2015kernel} unanimously suggest to differentiate log-determinants via
$\mu \coloneqq \trace{\log(A)}$ and $\diff \mu = \trace{A^{-1} \diff A}$ (\cref{equation-chain-rule-logdeterminant}).
Since we seem to be the first to take a different path, benchmarking Gaussian processes in comparison to GPyTorch is a good first testbed for our gradients.

\begin{table}[t]
\caption{Our method yields the same root-mean-square errors (RMSEs) as GPyTorch. It reaches lower training losses but is $\approx$ 20$\times$ slower per epoch due to different matrix-vector-product backends (see \Cref{appendix-section-low-memory-matrix-vector-products}). Three runs, significant improvements in bold.
 We use an 80/20 train/test split.}
\label{table-gaussian-process-results}
\begin{center}
\small
\begin{tabular}{ l c c c c c c} 
\toprule
Dataset 
	& Size
	& Dim.
	& Method 
	& RMSE ~$\downarrow$
	& Final training loss ~$\downarrow$
	& Runtime (s/epoch) ~$\downarrow$
	\\
\midrule
\multirow{2}{4.55em}{\texttt{elevators}}
	& \multirow{2}{2.0em}{16,599}
	& \multirow{2}{1.0em}{18}
		& Adjoints  
		& 0.09 $\pm$ 0.002
		& \cellcolor[gray]{0.9}{\bf -0.91 $\pm$ 0.025}
		& ~~1.69 $\pm$ 0.000
		\\ 
	&
	&
		& GPyTorch 
		& 0.09 $\pm$ 0.003
		& -0.63 $\pm$ 0.062	
		& ~~\cellcolor[gray]{0.9}{\bf 0.10 $\pm$ 0.004}
		\\ 
\midrule
\multirow{2}{4.55em}{\texttt{protein}}
	& \multirow{2}{2.0em}{45,730}
	& \multirow{2}{1.0em}{9}
		& Adjoints  
		& 0.39 $\pm$ 0.005
		& ~0.73 $\pm$ 0.300
		& 12.62 $\pm$ 0.172
		\\ 
	&
	&
		& GPyTorch 
		& 0.39 $\pm$ 0.005
		& ~0.73 $\pm$ 0.075
		& ~~\cellcolor[gray]{0.9}{\bf 0.73 $\pm$ 0.039}
		\\ 
\midrule
\multirow{2}{4.55em}{\texttt{kin40k}}
	& \multirow{2}{2.0em}{40,000}
	& \multirow{2}{1.0em}{8}
		& Adjoints  
		& 0.12 $\pm$ 0.004
		& \cellcolor[gray]{0.9}{\bf -0.30 $\pm$ 0.078}
		& ~~8.27 $\pm$ 0.004
		\\ 
	&
	&
		& GPyTorch 
		& 0.10 $\pm$ 0.010
		& -0.26 $\pm$ 0.094
		& ~~\cellcolor[gray]{0.9}{\bf 0.26 $\pm$ 0.024}
		\\ 
\midrule
\multirow{2}{4.55em}{\texttt{kegg\_dir}}
	& \multirow{2}{2.0em}{48,827}
	& \multirow{2}{1.0em}{20}
		& Adjoints  
		& 0.12 $\pm$ 0.002
		& -0.59 $\pm$ 0.295
		& 13.25 $\pm$ 0.005
		\\ 
	&
	&
		& GPyTorch 
		& 0.12 $\pm$ 0.005
		& -0.41 $\pm$  0.054
		& ~~\cellcolor[gray]{0.9}{\bf 0.62 $\pm$ 0.262}
		\\ 
\midrule
\multirow{2}{4.55em}{\texttt{kegg\_undir}}
	& \multirow{2}{2.0em}{63,608}
	& \multirow{2}{1.0em}{26}
		& Adjoints  
		& 0.12 $\pm$ 0.002
		& \cellcolor[gray]{0.9}{\bf -0.69 $\pm$ 0.263}
		& 24.20 $\pm$ 0.004
		\\ 
	&
	&
		& GPyTorch 
		& 0.12 $\pm$ 0.003
		& -0.40 $\pm$ 0.039
		& ~~\cellcolor[gray]{0.9}{\bf 1.67 $\pm$ 0.532}
		\\ 

\bottomrule
\end{tabular}
\end{center}
\end{table}
\paragraph{Setup: Like GPyTorch's defaults}
We mimic recent suggestions for scalable Gaussian process models \citep{wang2019exact,wenger2022preconditioning}: we implement a pivoted Cholesky preconditioner \citep{harbrecht2012low} and combine it with conjugate gradient solvers for $x^\top K^{-1} x$ (which can be differentiated efficiently). We estimate the log-determinant stochastically via $\log\det(A) = \trace{\log(A)} = E[v^\top \log(A) v]$, and compute $\log(A) v$ via the Lanczos iteration. While all of the above is common for ``exact'' Gaussian processes \citep{wenger2022preconditioning,gardner2018gpytorch,wang2019exact} (``exact'' as opposed to variational approaches, which are not relevant for this comparison), there are three key differences between our code and GPyTorch's:
(i) GPyTorch is in Pytorch and uses KeOps \citep{charlier2021kernel} for efficient kernel-matrix-vector products. We use JAX and must build our own low-memory matrix-vector products (\cref{appendix-section-low-memory-matrix-vector-products}).
(ii) GPyTorch runs all algorithms adaptively (we specify tolerances and maximum iterations as much as possible). We use adaptive conjugate gradient solvers and fixed ranks for everything else.
\textit{(iii) GPytorch differentiates the log-determinant with a tailored approximation of \cref{equation-chain-rule-logdeterminant} \citep{gardner2018gpytorch}; we embed our gradients of the Lanczos iteration into automatic differentiation.}
To keep the benchmark objective, we mimic the parameter suggestions from GPyTorch's default settings, and optimise hyperparameters of a Mat\'ern$\left(\frac{3}{2}\right)$ model on UCI datasets with the Adam optimiser \citep{KingBa15}.
\cref{appendix-experiment-details-gaussian-processes} lists parameters and discusses the datasets.

\paragraph{Analysis: Trains like GPyTorch; large scale only limited by matrix-vector-product backends}
In this benchmark, we are looking for low reconstruction errors, fast runtimes and well-behaved loss functions.
\cref{table-gaussian-process-results} shows that this is the case for both implementations: the reconstruction errors are essentially the same, and both methods converge well (we achieve lower training losses).
This result shows that by taking a numerically exact gradient of the Lanczos iteration, and leaving everything else to automatic differentiation, matches the performance of state-of-the-art solvers.
Larger datasets are only limited by the efficiency of our matrix-vector products (in comparison to KeOps); \cref{appendix-section-low-memory-matrix-vector-products} discusses this in detail.
Overall, this result strengthens the democratisation of exact Gaussian processes because it reveils a simple yet effective alternative to GPyTorch's domain-specific gradients.

\section{Case study: Physics-informed machine learning with PDEs}
\label{section-case-study-pde}

Much of the paper thus far discusses functions of matrices in the context of log-determinants.
So, in order to demonstrate performance for (i) a problem that is not a log-determinant and (ii) for a non-symmetric matrix which requires Arnoldi instead of Lanczos, we learn the coefficient field $\omega$ of
\begin{align}
\frac{\partial^2}{\partial t^2}u(t; x_1, x_2) = \omega(x_1, x_2)^2 \left[\frac{\partial^2}{\partial x_1^2} u(t; x_1, x_2) + \frac{\partial^2}{\partial x_2^2}u(t; x_1, x_2)\right]
\end{align}
subject to Neumann boundary conditions.
We discretise this equation on a $128\times128$ grid in space and transform the resulting $128^2$-dimensional second-order ordinary differential equation into a first-order differential equation, $\dot w = A w$, $w(0)= w_0$, with solution operator $w(t) = \exp(A t) w_0$.
The system matrix $A$ is sparse, asymmetric, and has $32,768$ rows and columns.
We sample a true $\omega$ from a Gaussian process with a square exponential kernel and generate data by sampling $256$ initial conditions and solving the equation numerically with high precision. Details are in \cref{appendix-section-pde-data}.

\paragraph{Setup: Arnoldi vs Diffrax's Runge-Kutta methods for a 250k parameter MLP}
We learn $\omega$ with a multi-layer perceptron (MLP) with approximately 250,000 parameters.
We had similar reconstructions with fewer parameters but use 250,000 to display how gradients of the Arnoldi iteration scale to many parameters.
We compare an implementation of the solution operator $(\theta, w_0) \mapsto \exp(A(\theta))w_0$ with the Arnoldi iteration to Diffrax's \citep{kidger2022neural} implementation of ``Dopri5'' \citep{dormand1980family,shampine1986some} 
\begin{wrapfigure}[14]{r}{0.44\linewidth}
\vspace{-0.5\baselineskip}
\includegraphics[width=\linewidth]{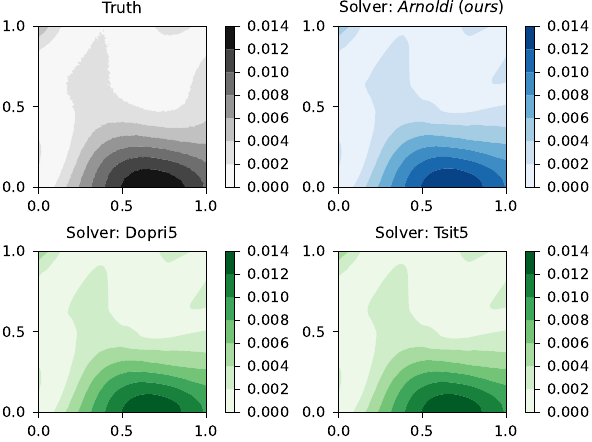}
\caption{All methods find the truth.}
\label{figure-pde-contours}
\end{wrapfigure}
with a differentiate-then-discretise adjoint \citep{chen2018neural}
as well as ``Tsit5'' \citep{tsitouras2011runge} with a discretise-then-differentiate adjoint (recommended by \citep{kidger2022neural,rackauckas2017differentialequations}).
All methods receive equal matrix-vector products per simulation.

\paragraph{Analysis: All methods train, but Arnoldi is more accurate for fixed matrix-vector-product budgets}
We evaluate the approximation errors in computing the values and gradients of a mean-squared error loss for all three solvers and then use the solvers to train the MLP.
We are looking for low approximation errors for few matrix-vector products and for a good reconstruction of the truth.
\cref{figure-pde-errors-all} shows the results. 
\begin{figure}[t]
\centering
\includegraphics[width=\linewidth]{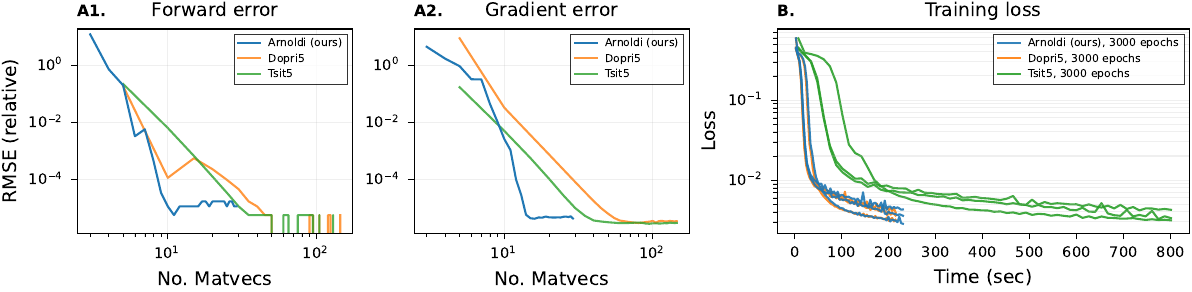}
\caption{Arnoldi's superior convergence on the forward pass (A1) is inherited by the gradients (A2; mind the shared $y$-axis) and ultimately leads to fast training (B). For training, Arnoldi uses ten matrix-vector products, and the other two use 15 (so they have equal error $\approx 10^{-4}$ in A1 and A2.)}
\label{figure-pde-errors-all}
\end{figure}
\begin{table}[t]
\caption{All three methods reconstruct the parameter well (std.-deviations exceed differences for test-loss and RMSE), but Arnoldi and Dopri5 are faster than Tsit5. 
Dopri5 uses the \texttt{BacksolveAdjoint}, and Tsit5 the \texttt{RecursiveCheckpointAdjoint} in Diffrax \citep{kidger2022neural}.
We contribute Arnoldi's adjoints.}
\label{table-pde-results}
\begin{center}
\begin{tabular}{lccc}
\toprule
 & Arnoldi (adjoints; ours) & Dopri5 (diff. $\rightarrow$ disc.)  & Tsit5 (disc. $\rightarrow$ diff.) \\
\midrule
Loss on test set 
	& 6.1e-03 $\pm$ 3.3e-04 
	& 6.3e-03 $\pm$ 5.7e-04  
	& 5.9e-03 $\pm$ 2.2e-04
	\\
Parameter RMSE 
	& 2.9e-04 $\pm$ 4.4e-05
	& 2.6e-04 $\pm$ 5.0e-05 
	& 2.7e-04 $\pm$ 5.2e-05
	\\
Runtime per epoch 
	& \cellcolor[gray]{0.9}{\bf 7.7e-02 $\pm$ 1.8e-05}
	& \cellcolor[gray]{0.9}{\bf 7.2e-02 $\pm$ 3.4e-05}
	& 2.7e-01 $\pm$ 1.1e-05
	\\
\bottomrule
\end{tabular}
\end{center}
\end{table}
The Arnoldi iteration has the lowest forward-pass and gradient error, but \cref{table-pde-results} demonstrates how all approaches lead to low errors on $\omega$ as well as on a test set (a held-back percentage of the training data); see also \cref{figure-pde-contours}.
The adjoints of the Arnoldi iteration match the efficiency of the differentiate-then-discretise adjoint \citep{chen2018neural}, and both outperform the discretise-then-differentiate adjoint by a margin.
This shows how linear-algebra solutions to matrix exponentials can compete with highly optimised differential equation solvers. We anticipate ample opportunities of using the now-differentiable Arnoldi iteration for physics-based machine learning.

\section{Case study: Calibrating Bayesian neural networks}
\label{section-case-study-bayesian-neural-networks}
Next, we differentiate a function of a matrix on a problem that is native to machine learning: marginal likelihood optimisation of a Bayesian neural network (a high-level introduction is in \Cref{appendix-laplace}).

\paragraph{Setup: Laplace-approximation of a VAN pre-trained on ImageNet}
We consider as $g_\theta(x)$ a ``Visual Attention Network'' \citep{guo2023visual} with 4,105,800 parameters, pre-trained on ImageNet \citep{deng2009imagenet}.
We assume $p(\theta) = N(0, \alpha^{-2}I)$, and Laplace-approximate the log-marginal likelihood of the data as
\begin{align}\label{equation-marginal-likelihood-laplace}
\log p(y \mid x) \approx \log p(y, \theta \mid x) - \frac{1}{2} \log\det(A(\alpha)) + \text{const}
\end{align}
where $A(\alpha)$ is the generalised Gauss--Newton matrix (GGN) from \Cref{section-introduction} (recall \Cref{equation-gauss-newton-matrix}).
\begin{wrapfigure}[14]{r}{0.30\linewidth}
\vspace{-1.0\baselineskip}
\centering
\includegraphics[width=\linewidth]{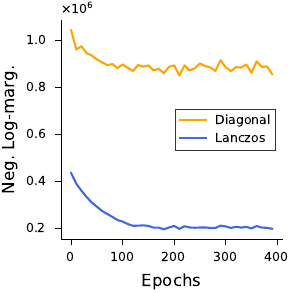}
\caption{Lanczos vs diagonal approx. for a Bayesian VAN.}
\label{figure-training-for-laplace}
\end{wrapfigure}
We optimise $\alpha$ via \Cref{equation-marginal-likelihood-laplace}, implementing the log-determinant via stochastic trace estimation in combination with a Lanczos iteration (like in \Cref{section-case-study-gaussian-processes}).
Contemporary works \citep{daxberger2021bayesian,kristiadi2020being,snoek2015scalable,denker1990transforming,ritter2018scalable,ritter2018online,maddox2020rethinking,sharma2021sketching} rely on sparse approximations of the GGN (such as diagonal or KFAC approximations), so we compare our implementation to a diagonal approximation of the GGN matrix, which yields closed-form log-determinants.
The exact diagonal of the 4-million-column GGN matrix would require 4 million GGN-vector products with unit vectors, and like \citet{deng2022accelerated}, we find this too expensive and resort to stochastic diagonal approximation (similar to trace estimation; all details are in \Cref{appendix-laplace}).
We give both the stochastic diagonal approximation and our Lanczos-based estimator exactly 150 matrix-vector products to approximate \cref{equation-marginal-likelihood-laplace}.
We compare the evolution of the loss function over time and
 various uncertainty calibration metrics.
\Cref{figure-training-for-laplace} demonstrates training and \Cref{table-laplace-results} shows results.
\begin{table}[t]
\caption{Lanczos outperforms the diagonal approximation for calibrating a Bayesian version of an ImageNet pre-trained VAN.
One training run; calibration estimated with 30 samples (sampling ``Lanczos'' and ``diagonal'' with another Lanczos/diagonal approximation; see \cref{appendix-laplace}).
\texttt{places365} \citep{zhou2014learning} is the out-of-distribution data; mean and std.-deviations of 3 runs. ``MVs'': matrix-vector products.
}
\label{table-laplace-results}
\begin{center}
\small
\begin{tabular}{lrrrr}
    \toprule
    	& 
    	&  Lanczos (50 MVs) 
    	&  Lanczos (150 MVs) 
    	& Diagonal (150 MVs)
    	\\
    \midrule
    Runtime (sec)
    	& $\downarrow$
    	& \cellcolor[gray]{0.9}\bf 950
    	& 4674
    	& 4314 
    	\\
    Marginal likelihood (log)
    	& $\uparrow$
    	& -192,757
    	&  \cellcolor[gray]{0.9} \bf -154,475
    	& -876,444 
    	\\
	\midrule
    Joint likelihood: train (log)
    	& $\uparrow$
    	& \cellcolor[gray]{0.9}\bf -81.2 $\pm$ 15.4 
    	&  \cellcolor[gray]{0.9}\bf  -81.2 $\pm$ 15.4 
    	& -5,669.2 $\pm$ 124.1 
    	\\
    Joint likelihood: test (log)
    	& $\uparrow$
    	& \cellcolor[gray]{0.9}\bf -66.5 $\pm$ 11.2 
    	& \cellcolor[gray]{0.9}\bf  -66.5 $\pm$ 11.2 
    	& -5,260.9 $\pm$ 290.2 
    	\\
    Expected calibration error
    	& $\downarrow$
    	& 0.5  $\pm$ 0.01
    	& 0.5  $\pm$ 0.01
    	& \cellcolor[gray]{0.9}\bf 0.2 $\pm$ 0.003 
    	\\
    AUROC (out-of-dist.)
    	& $\uparrow$
    	&\cellcolor[gray]{0.9}{\bf 0.9  $\pm$ 0.03}
    	& \cellcolor[gray]{0.9}{\bf  0.9  $\pm$ 0.03}
    	& 0.5 $\pm$ 0.010
    	\\
    \bottomrule
\end{tabular}
\end{center}
\end{table}

\textbf{Analysis: Lanczos uses matrix-vector products better (by a margin) }
The results suggest how, for a fixed matrix-vector-product budget, Lanczos achieves a drastically better likelihood at a similar computational budget and already shows significant improvement with a much smaller budget. 
Lanczos outperforms the diagonal approximation on all metrics except ECE.
The subpar performance of the diagonal approximation matches the observations of \citet{ritter2018scalable}; see also \citep{daxberger2021laplace}. 
The main takeaway from this study is that differentiable matrix-free linear algebra unlocks new techniques for Laplace approximations and allows further advances for Bayesian neural networks in general.

\newpage
\acksection{
This work was supported by a research grant (42062) from VILLUM FONDEN. The work was partly
funded by the Novo Nordisk Foundation through the Center for Basic Machine Learning Research in Life Science (NNF20OC0062606).
This project received funding from the European Research Council (ERC) under the European Union’s Horizon programme (grant agreement 101125993).
}

\medskip
\bibliography{bibfile}

\newcommand{\arxiv}[1]{Preprint on ArXiv:#1}
\begin{thebibliography}{105}
\providecommand{\natexlab}[1]{#1}
\providecommand{\url}[1]{\texttt{#1}}
\expandafter\ifx\csname urlstyle\endcsname\relax
  \providecommand{\doi}[1]{doi: #1}\else
  \providecommand{\doi}{doi: \begingroup \urlstyle{rm}\Url}\fi

\bibitem[Higham(2008)]{higham2008functions}
Nicholas~J Higham.
\newblock \emph{Functions of Matrices: Theory and Computation}.
\newblock SIAM, 2008.

\bibitem[Gallopoulos and Saad(1992)]{gallopoulos1992efficient}
Efstratios Gallopoulos and Yousef Saad.
\newblock Efficient solution of parabolic equations by {Krylov} approximation
  methods.
\newblock \emph{SIAM Journal on Scientific and Statistical Computing},
  13\penalty0 (5):\penalty0 1236--1264, 1992.

\bibitem[Hochbruck and Ostermann(2010)]{hochbruck2010exponential}
Marlis Hochbruck and Alexander Ostermann.
\newblock Exponential integrators.
\newblock \emph{Acta Numerica}, 19:\penalty0 209--286, 2010.

\bibitem[Xiao and Liu(2020)]{xiao2020generative}
Changyi Xiao and Ligang Liu.
\newblock Generative flows with matrix exponential.
\newblock In \emph{International Conference on Machine Learning}, pages
  10452--10461. PMLR, 2020.

\bibitem[Axelsson and Gustafsson(2014)]{axelsson2014discrete}
Patrik Axelsson and Fredrik Gustafsson.
\newblock Discrete-time solutions to the continuous-time differential
  {Lyapunov} equation with applications to {Kalman} filtering.
\newblock \emph{IEEE Transactions on Automatic Control}, 60\penalty0
  (3):\penalty0 632--643, 2014.

\bibitem[Gardner et~al.(2018)Gardner, Pleiss, Weinberger, Bindel, and
  Wilson]{gardner2018gpytorch}
Jacob Gardner, Geoff Pleiss, Kilian~Q Weinberger, David Bindel, and Andrew~G
  Wilson.
\newblock {GPyTorch}: {Blackbox} matrix-matrix {Gaussian} process inference
  with {GPU} acceleration.
\newblock In \emph{Advances in Neural Information Processing Systems}, 2018.

\bibitem[Wenger et~al.(2022)Wenger, Pleiss, Hennig, Cunningham, and
  Gardner]{wenger2022preconditioning}
Jonathan Wenger, Geoff Pleiss, Philipp Hennig, John Cunningham, and Jacob
  Gardner.
\newblock Preconditioning for scalable {Gaussian} process hyperparameter
  optimization.
\newblock In \emph{International Conference on Machine Learning}. PMLR, 2022.

\bibitem[Immer et~al.(2023)Immer, Van Der~Ouderaa, Van Der~Wilk, Ratsch, and
  Sch{\"o}lkopf]{immer2023stochastic}
Alexander Immer, Tycho~FA Van Der~Ouderaa, Mark Van Der~Wilk, Gunnar Ratsch,
  and Bernhard Sch{\"o}lkopf.
\newblock Stochastic marginal likelihood gradients using neural tangent
  kernels.
\newblock In \emph{International Conference on Machine Learning}, pages
  14333--14352. PMLR, 2023.

\bibitem[Behrmann et~al.(2019)Behrmann, Grathwohl, Chen, Duvenaud, and
  Jacobsen]{behrmann2019invertible}
Jens Behrmann, Will Grathwohl, Ricky~TQ Chen, David Duvenaud, and
  J{\"o}rn-Henrik Jacobsen.
\newblock Invertible residual networks.
\newblock In \emph{International Conference on Machine Learning}, pages
  573--582. PMLR, 2019.

\bibitem[Chen et~al.(2019)Chen, Behrmann, Duvenaud, and
  Jacobsen]{chen2019residual}
Ricky~TQ Chen, Jens Behrmann, David~K Duvenaud, and J{\"o}rn-Henrik Jacobsen.
\newblock Residual flows for invertible generative modeling.
\newblock In \emph{Advances in Neural Information Processing Systems},
  volume~32, 2019.

\bibitem[Pleiss et~al.(2018)Pleiss, Gardner, Weinberger, and
  Wilson]{pleiss2018constant}
Geoff Pleiss, Jacob Gardner, Kilian Weinberger, and Andrew~Gordon Wilson.
\newblock Constant-time predictive distributions for {Gaussian} processes.
\newblock In \emph{International Conference on Machine Learning}, pages
  4114--4123. PMLR, 2018.

\bibitem[Finlay et~al.(2020)Finlay, Jacobsen, Nurbekyan, and
  Oberman]{finlay2020train}
Chris Finlay, J{\"o}rn-Henrik Jacobsen, Levon Nurbekyan, and Adam Oberman.
\newblock How to train your neural {ODE}: the world of {Jacobian} and kinetic
  regularization.
\newblock In \emph{International Conference on Machine Learning}, pages
  3154--3164. PMLR, 2020.

\bibitem[Schraudolph(2002)]{schraudolph2002fast}
Nicol~N Schraudolph.
\newblock Fast curvature matrix-vector products for second-order gradient
  descent.
\newblock \emph{Neural Computation}, 14\penalty0 (7):\penalty0 1723--1738,
  2002.

\bibitem[Daxberger et~al.(2021{\natexlab{a}})Daxberger, Kristiadi, Immer,
  Eschenhagen, Bauer, and Hennig]{daxberger2021laplace}
Erik Daxberger, Agustinus Kristiadi, Alexander Immer, Runa Eschenhagen,
  Matthias Bauer, and Philipp Hennig.
\newblock {Laplace} redux -- effortless {Bayesian} deep learning.
\newblock \emph{Advances in Neural Information Processing Systems},
  34:\penalty0 20089--20103, 2021{\natexlab{a}}.

\bibitem[Hutchinson(1989)]{hutchinson1989stochastic}
Michael~F Hutchinson.
\newblock A stochastic estimator of the trace of the influence matrix for
  {Laplacian} smoothing splines.
\newblock \emph{Communications in Statistics-Simulation and Computation},
  18\penalty0 (3):\penalty0 1059--1076, 1989.

\bibitem[Lanczos(1950)]{lanczos1950iteration}
Cornelius Lanczos.
\newblock An iteration method for the solution of the eigenvalue problem of
  linear differential and integral operators.
\newblock 1950.

\bibitem[Ubaru et~al.(2017)Ubaru, Chen, and Saad]{ubaru2017fast}
Shashanka Ubaru, Jie Chen, and Yousef Saad.
\newblock Fast estimation of $\mathrm{tr}(f({A}))$ via stochastic {Lanczos}
  quadrature.
\newblock \emph{SIAM Journal on Matrix Analysis and Applications}, 38\penalty0
  (4):\penalty0 1075--1099, 2017.

\bibitem[Chen and Hallman(2023)]{chen2023krylov}
Tyler Chen and Eric Hallman.
\newblock {Krylov}-aware stochastic trace estimation.
\newblock \emph{SIAM Journal on Matrix Analysis and Applications}, 44\penalty0
  (3):\penalty0 1218--1244, 2023.

\bibitem[Wang et~al.(2019)Wang, Pleiss, Gardner, Tyree, Weinberger, and
  Wilson]{wang2019exact}
Ke~Wang, Geoff Pleiss, Jacob Gardner, Stephen Tyree, Kilian~Q Weinberger, and
  Andrew~Gordon Wilson.
\newblock Exact {Gaussian} processes on a million data points.
\newblock \emph{Advances in Neural Information Processing Systems}, 32, 2019.

\bibitem[Wu et~al.(2024)Wu, Wenger, Jones, Pleiss, and Gardner]{wu2023large}
Kaiwen Wu, Jonathan Wenger, Haydn Jones, Geoff Pleiss, and Jacob~R Gardner.
\newblock Large-scale {Gaussian} processes via alternating projection.
\newblock \emph{International Conference on Artificial Intelligence and
  Statistics}, 2024.

\bibitem[K{\"o}nig et~al.(2023)K{\"o}nig, Pfeffer, and
  Stoll]{konig2023efficient}
Josie K{\"o}nig, Max Pfeffer, and Martin Stoll.
\newblock Efficient training of {Gaussian} processes with tensor product
  structure.
\newblock \emph{\arxiv{2312.15305}}, 2023.

\bibitem[Davies(2015)]{davies2015effective}
Alexander~James Davies.
\newblock \emph{Effective implementation of {Gaussian} process regression for
  machine learning}.
\newblock PhD thesis, University of Cambridge, 2015.

\bibitem[Dong et~al.(2017)Dong, Eriksson, Nickisch, Bindel, and
  Wilson]{dong2017scalable}
Kun Dong, David Eriksson, Hannes Nickisch, David Bindel, and Andrew~G Wilson.
\newblock Scalable log determinants for {Gaussian} process kernel learning.
\newblock In \emph{Advances in Neural Information Processing Systems}, 2017.

\bibitem[Pleiss et~al.(2020)Pleiss, Jankowiak, Eriksson, Damle, and
  Gardner]{pleiss2020fast}
Geoff Pleiss, Martin Jankowiak, David Eriksson, Anil Damle, and Jacob Gardner.
\newblock Fast matrix square roots with applications to {Gaussian} processes
  and {Bayesian} optimization.
\newblock \emph{Advances in Neural Information Processing Systems},
  33:\penalty0 22268--22281, 2020.

\bibitem[Wang(2015)]{wang2015krylov}
Hao Wang.
\newblock The {Krylov} subspace methods for the computation of matrix
  exponentials.
\newblock 2015.

\bibitem[Zhang and Chen(2023)]{zhang2022fast}
Qinsheng Zhang and Yongxin Chen.
\newblock Fast sampling of diffusion models with exponential integrator.
\newblock \emph{International Conference on Learning Representations}, 2023.

\bibitem[Rezende and Mohamed(2015)]{rezende2015variational}
Danilo Rezende and Shakir Mohamed.
\newblock Variational inference with normalizing flows.
\newblock In \emph{International Conference on Machine Learning}, pages
  1530--1538. PMLR, 2015.

\bibitem[Hoogeboom et~al.(2020)Hoogeboom, Garcia~Satorras, Tomczak, and
  Welling]{hoogeboom2020convolution}
Emiel Hoogeboom, Victor Garcia~Satorras, Jakub Tomczak, and Max Welling.
\newblock The convolution exponential and generalized {Sylvester} flows.
\newblock In \emph{Advances in Neural Information Processing Systems},
  volume~33, pages 18249--18260, 2020.

\bibitem[Hochbruck et~al.(1998)Hochbruck, Lubich, and
  Selhofer]{hochbruck1998exponential}
Marlis Hochbruck, Christian Lubich, and Hubert Selhofer.
\newblock Exponential integrators for large systems of differential equations.
\newblock \emph{SIAM Journal on Scientific Computing}, 19\penalty0
  (5):\penalty0 1552--1574, 1998.

\bibitem[Immer et~al.(2021{\natexlab{a}})Immer, Bauer, Fortuin, R{\"a}tsch, and
  Emtiyaz]{immer2021scalable}
Alexander Immer, Matthias Bauer, Vincent Fortuin, Gunnar R{\"a}tsch, and
  Khan~Mohammad Emtiyaz.
\newblock Scalable marginal likelihood estimation for model selection in deep
  learning.
\newblock In \emph{International Conference on Machine Learning},
  2021{\natexlab{a}}.

\bibitem[Ritter et~al.(2018{\natexlab{a}})Ritter, Botev, and
  Barber]{ritter2018scalable}
Hippolyt Ritter, Aleksandar Botev, and David Barber.
\newblock A scalable {Laplace} approximation for neural networks.
\newblock In \emph{International Conference on Learning Representations},
  volume~6, 2018{\natexlab{a}}.

\bibitem[Golub and Meurant(2009)]{golub2009matrices}
Gene~H Golub and G{\'e}rard Meurant.
\newblock \emph{Matrices, Moments and Quadrature with Applications}, volume~30.
\newblock Princeton University Press, 2009.

\bibitem[Lipman et~al.(2023)Lipman, Chen, Ben-Hamu, Nickel, and
  Le]{lipman2022flow}
Yaron Lipman, Ricky~TQ Chen, Heli Ben-Hamu, Maximilian Nickel, and Matthew Le.
\newblock Flow matching for generative modeling.
\newblock In \emph{International Conference on Learning Representations}, 2023.

\bibitem[Ramesh and LeCun(2018)]{ramesh2018backpropagation}
Aditya Ramesh and Yann LeCun.
\newblock Backpropagation for implicit spectral densities.
\newblock In \emph{Advances in Neural Information Processing Systems}, 2018.

\bibitem[Williams and Rasmussen(2006)]{williams2006gaussian}
Christopher~KI Williams and Carl~Edward Rasmussen.
\newblock \emph{{Gaussian} Processes for Machine Learning}, volume~2.
\newblock MIT Press Cambridge, MA, 2006.

\bibitem[Cutajar et~al.(2016)Cutajar, Osborne, Cunningham, and
  Filippone]{cutajar2016preconditioning}
Kurt Cutajar, Michael Osborne, John Cunningham, and Maurizio Filippone.
\newblock Preconditioning kernel matrices.
\newblock In \emph{International Conference on Machine Learning}, pages
  2529--2538. PMLR, 2016.

\bibitem[Daxberger et~al.(2021{\natexlab{b}})Daxberger, Nalisnick, Allingham,
  Antor{\'a}n, and Hern{\'a}ndez-Lobato]{daxberger2021bayesian}
Erik Daxberger, Eric Nalisnick, James~U Allingham, Javier Antor{\'a}n, and
  Jos{\'e}~Miguel Hern{\'a}ndez-Lobato.
\newblock {Bayesian} deep learning via subnetwork inference.
\newblock In \emph{International Conference on Machine Learning}, pages
  2510--2521. PMLR, 2021{\natexlab{b}}.

\bibitem[Kristiadi et~al.(2020)Kristiadi, Hein, and Hennig]{kristiadi2020being}
Agustinus Kristiadi, Matthias Hein, and Philipp Hennig.
\newblock Being {Bayesian}, even just a bit, fixes overconfidence in {ReLu}
  networks.
\newblock In \emph{International Conference on Machine Learning}, pages
  5436--5446. PMLR, 2020.

\bibitem[Snoek et~al.(2015)Snoek, Rippel, Swersky, Kiros, Satish, Sundaram,
  Patwary, Prabhat, and Adams]{snoek2015scalable}
Jasper Snoek, Oren Rippel, Kevin Swersky, Ryan Kiros, Nadathur Satish,
  Narayanan Sundaram, Mostofa Patwary, Mr~Prabhat, and Ryan Adams.
\newblock Scalable {Bayesian} optimization using deep neural networks.
\newblock In \emph{International Conference on Machine Learning}, pages
  2171--2180. PMLR, 2015.

\bibitem[Denker and LeCun(1990)]{denker1990transforming}
John Denker and Yann LeCun.
\newblock Transforming neural-net output levels to probability distributions.
\newblock \emph{Advances in Neural Information Processing Systems}, 3, 1990.

\bibitem[Ritter et~al.(2018{\natexlab{b}})Ritter, Botev, and
  Barber]{ritter2018online}
Hippolyt Ritter, Aleksandar Botev, and David Barber.
\newblock Online structured {Laplace} approximations for overcoming
  catastrophic forgetting.
\newblock \emph{Advances in Neural Information Processing Systems}, 31,
  2018{\natexlab{b}}.

\bibitem[Maddox et~al.(2020)Maddox, Benton, and Wilson]{maddox2020rethinking}
Wesley~J Maddox, Gregory Benton, and Andrew~Gordon Wilson.
\newblock Rethinking parameter counting in deep models: Effective
  dimensionality revisited.
\newblock \emph{arXiv preprint arXiv:2003.02139}, 2020.

\bibitem[Sharma et~al.(2021)Sharma, Azizan, and Pavone]{sharma2021sketching}
Apoorva Sharma, Navid Azizan, and Marco Pavone.
\newblock Sketching curvature for efficient out-of-distribution detection for
  deep neural networks.
\newblock In \emph{Uncertainty in Artificial Intelligence}, pages 1958--1967.
  PMLR, 2021.

\bibitem[Miani et~al.(2022)Miani, Warburg, Moreno-Muñoz, Detlefsen, and
  Hauberg]{miani2022laplacian}
Marco Miani, Frederik Warburg, Pablo Moreno-Muñoz, Nicki~Skafte Detlefsen, and
  S{\o}ren Hauberg.
\newblock {Laplacian} autoencoders for learning stochastic representations.
\newblock In \emph{Advances in Neural Information Processing Systems}, 2022.

\bibitem[Hochbruck and Lubich(1997)]{hochbruck1997krylov}
Marlis Hochbruck and Christian Lubich.
\newblock On {Krylov} subspace approximations to the matrix exponential
  operator.
\newblock \emph{SIAM Journal on Numerical Analysis}, 34\penalty0 (5):\penalty0
  1911--1925, 1997.

\bibitem[S{\"a}rkk{\"a} and Solin(2019)]{sarkka2019applied}
Simo S{\"a}rkk{\"a} and Arno Solin.
\newblock \emph{Applied Stochastic Differential Equations}, volume~10.
\newblock Cambridge University Press, 2019.

\bibitem[Rissanen et~al.(2023)Rissanen, Heinonen, and
  Solin]{rissanen2022generative}
Severi Rissanen, Markus Heinonen, and Arno Solin.
\newblock Generative modelling with inverse heat dissipation.
\newblock In \emph{International Conference on Learning Representations}, 2023.

\bibitem[Moler and Van~Loan(1978)]{moler1978nineteen}
Cleve Moler and Charles Van~Loan.
\newblock Nineteen dubious ways to compute the exponential of a matrix.
\newblock \emph{SIAM Review}, 20\penalty0 (4):\penalty0 801--836, 1978.

\bibitem[Moler and Van~Loan(2003)]{moler2003nineteen}
Cleve Moler and Charles Van~Loan.
\newblock Nineteen dubious ways to compute the exponential of a matrix,
  twenty-five years later.
\newblock \emph{SIAM Review}, 45\penalty0 (1):\penalty0 3--49, 2003.

\bibitem[Al-Mohy and Higham(2009)]{al2009computing}
Awad~H Al-Mohy and Nicholas~J Higham.
\newblock Computing the {Fr{\'e}chet} derivative of the matrix exponential,
  with an application to condition number estimation.
\newblock \emph{SIAM Journal on Matrix Analysis and Applications}, 30\penalty0
  (4):\penalty0 1639--1657, 2009.

\bibitem[Cao et~al.(2003)Cao, Li, Petzold, and Serban]{cao2003adjoint}
Yang Cao, Shengtai Li, Linda Petzold, and Radu Serban.
\newblock Adjoint sensitivity analysis for differential-algebraic equations:
  The adjoint {DAE} system and its numerical solution.
\newblock \emph{SIAM Journal on Scientific Computing}, 24\penalty0
  (3):\penalty0 1076--1089, 2003.

\bibitem[Diehl and Gros(2011)]{diehl2011numerical}
Moritz Diehl and S{\'e}bastien Gros.
\newblock Numerical optimal control.
\newblock \emph{Optimization in Engineering Center (OPTEC)}, 2011.

\bibitem[Rackauckas and Nie(2017)]{rackauckas2017differentialequations}
Christopher Rackauckas and Qing Nie.
\newblock {DifferentialEquations}.jl--a performant and feature-rich ecosystem
  for solving differential equations in {Julia}.
\newblock \emph{Journal of Open Research Software}, 5\penalty0 (1):\penalty0
  15--15, 2017.

\bibitem[Kidger(2021)]{kidger2022neural}
Patrick Kidger.
\newblock \emph{{O}n {N}eural {D}ifferential {E}quations}.
\newblock PhD thesis, University of Oxford, 2021.

\bibitem[Rackauckas et~al.(2019)Rackauckas, Innes, Ma, Bettencourt, White, and
  Dixit]{rackauckas2019diffeqflux}
Chris Rackauckas, Mike Innes, Yingbo Ma, Jesse Bettencourt, Lyndon White, and
  Vaibhav Dixit.
\newblock {DiffEqFlux}.jl--a julia library for neural differential equations.
\newblock \emph{arXiv preprint arXiv:1902.02376}, 2019.

\bibitem[Ma et~al.(2021)Ma, Dixit, Innes, Guo, and
  Rackauckas]{ma2021comparison}
Yingbo Ma, Vaibhav Dixit, Michael~J Innes, Xingjian Guo, and Chris Rackauckas.
\newblock A comparison of automatic differentiation and continuous sensitivity
  analysis for derivatives of differential equation solutions.
\newblock In \emph{2021 IEEE High Performance Extreme Computing Conference
  (HPEC)}, pages 1--9. IEEE, 2021.

\bibitem[Golub and Van~Loan(2013)]{golub2013matrix}
Gene~H Golub and Charles~F Van~Loan.
\newblock \emph{Matrix Computations}.
\newblock JHU press, 2013.

\bibitem[Hestenes et~al.(1952)Hestenes, Stiefel, et~al.]{hestenes1952methods}
Magnus~Rudolph Hestenes, Eduard Stiefel, et~al.
\newblock \emph{Methods of conjugate gradients for solving linear systems},
  volume~49.
\newblock NBS Washington, DC, 1952.

\bibitem[Arnoldi(1951)]{arnoldi1951principle}
Walter~Edwin Arnoldi.
\newblock The principle of minimized iterations in the solution of the matrix
  eigenvalue problem.
\newblock \emph{Quarterly of Applied Mathematics}, 9\penalty0 (1):\penalty0
  17--29, 1951.

\bibitem[Seeger et~al.(2017)Seeger, Hetzel, Dai, Meissner, and
  Lawrence]{seeger2017auto}
Matthias Seeger, Asmus Hetzel, Zhenwen Dai, Eric Meissner, and Neil~D Lawrence.
\newblock Auto-differentiating linear algebra.
\newblock \emph{\arxiv{1710.08717}}, 2017.

\bibitem[Walter and Lehmann(2010)]{walter2010algorithmic}
Sebastian~F Walter and Lutz Lehmann.
\newblock Algorithmic differentiation of linear algebra functions with
  application in optimum experimental design (extended version).
\newblock \emph{\arxiv{1001.1654}}, 2010.

\bibitem[Griewank and Walther(2008)]{griewank2008evaluating}
Andreas Griewank and Andrea Walther.
\newblock \emph{Evaluating Derivatives: Principles and Techniques of
  Algorithmic Differentiation}.
\newblock SIAM, 2008.

\bibitem[Kolodziej et~al.(2019)Kolodziej, Aznaveh, Bullock, David, Davis,
  Henderson, Hu, and Sandstrom]{kolodziej2019suitesparse}
Scott~P Kolodziej, Mohsen Aznaveh, Matthew Bullock, Jarrett David, Timothy~A
  Davis, Matthew Henderson, Yifan Hu, and Read Sandstrom.
\newblock The {SuiteSparse} matrix collection website interface.
\newblock \emph{Journal of Open Source Software}, 4\penalty0 (35):\penalty0
  1244, 2019.

\bibitem[Davis and Hu(2011)]{davis2011university}
Timothy~A Davis and Yifan Hu.
\newblock The {University of Florida} sparse matrix collection.
\newblock \emph{ACM Transactions on Mathematical Software (TOMS)}, 38\penalty0
  (1):\penalty0 1--25, 2011.

\bibitem[Duff et~al.(1989)Duff, Grimes, and Lewis]{duff1989sparse}
Iain~S Duff, Roger~G Grimes, and John~G Lewis.
\newblock Sparse matrix test problems.
\newblock \emph{ACM Transactions on Mathematical Software (TOMS)}, 15\penalty0
  (1):\penalty0 1--14, 1989.

\bibitem[Blondel and Roulet(2024)]{blondel2024elements}
Mathieu Blondel and Vincent Roulet.
\newblock The elements of differentiable programming.
\newblock \emph{arXiv preprint arXiv:2403.14606}, 2024.

\bibitem[Groenenboom and Buck(1990)]{groenenboom1990solving}
Gerrit~C Groenenboom and Henk~M Buck.
\newblock Solving the discretized time-independent {Schr{\"o}dinger} equation
  with the {Lanczos} procedure.
\newblock \emph{The Journal of Chemical Physics}, 92\penalty0 (7):\penalty0
  4374--4379, 1990.

\bibitem[Paige(1971)]{paige1971computation}
Christopher~Conway Paige.
\newblock \emph{The computation of eigenvalues and eigenvectors of very large
  sparse matrices.}
\newblock PhD thesis, University of London, 1971.

\bibitem[Paige(1972)]{paige1972computational}
Christopher~C Paige.
\newblock Computational variants of the {Lanczos} method for the eigenproblem.
\newblock \emph{IMA Journal of Applied Mathematics}, 10\penalty0 (3):\penalty0
  373--381, 1972.

\bibitem[Paige(1976)]{paige1976error}
Christopher~C Paige.
\newblock Error analysis of the {Lanczos} algorithm for tridiagonalizing a
  symmetric matrix.
\newblock \emph{IMA Journal of Applied Mathematics}, 18\penalty0 (3):\penalty0
  341--349, 1976.

\bibitem[B{\"o}rm and Mehl(2012)]{borm2012numerical}
Steffen B{\"o}rm and Christian Mehl.
\newblock \emph{Numerical methods for Eigenvalue Problems}.
\newblock Walter de Gruyter, 2012.

\bibitem[Blondel et~al.(2022)Blondel, Berthet, Cuturi, Frostig, Hoyer,
  Llinares-L{\'o}pez, Pedregosa, and Vert]{blondel2022efficient}
Mathieu Blondel, Quentin Berthet, Marco Cuturi, Roy Frostig, Stephan Hoyer,
  Felipe Llinares-L{\'o}pez, Fabian Pedregosa, and Jean-Philippe Vert.
\newblock Efficient and modular implicit differentiation.
\newblock \emph{Advances in Neural Information Processing Systems},
  35:\penalty0 5230--5242, 2022.

\bibitem[Betancourt et~al.(2020)Betancourt, Margossian, and
  Leos-Barajas]{betancourt2020discrete}
Michael Betancourt, Charles~C Margossian, and Vianey Leos-Barajas.
\newblock The discrete adjoint method: Efficient derivatives for functions of
  discrete sequences.
\newblock \emph{\arxiv{2002.00326}}, 2020.

\bibitem[Tran and Leok(2024)]{tran2024geometric}
Brian~Kha Tran and Melvin Leok.
\newblock Geometric methods for adjoint systems.
\newblock \emph{Journal of Nonlinear Science}, 34\penalty0 (1):\penalty0 25,
  2024.

\bibitem[Roberts and Roberts(2020)]{roberts2020qr}
Denisa~AO Roberts and Lucas~R Roberts.
\newblock {QR} and {LQ} decomposition matrix backpropagation algorithms for
  square, wide, and deep--real or complex--matrices and their software
  implementation.
\newblock \emph{\arxiv{2009.10071}}, 2020.

\bibitem[Minka(2000)]{minka2000old}
Thomas~P Minka.
\newblock Old and new matrix algebra useful for statistics.
\newblock \emph{See www.stat.cmu.edu/minka/papers/matrix.html}, 4, 2000.

\bibitem[Giles(2008)]{giles2008extended}
Mike Giles.
\newblock An extended collection of matrix derivative results for forward and
  reverse mode automatic differentiation.
\newblock 2008.

\bibitem[C{\'e}a(1986)]{cea1986conception}
Jean C{\'e}a.
\newblock Conception optimale ou identification de formes, calcul rapide de la
  d{\'e}riv{\'e}e directionnelle de la fonction co{\^u}t.
\newblock \emph{ESAIM: Mod{\'e}lisation math{\'e}matique et analyse
  num{\'e}rique}, 20\penalty0 (3):\penalty0 371--402, 1986.

\bibitem[Bradbury et~al.(2018)Bradbury, Frostig, Hawkins, Johnson, Leary,
  Maclaurin, Necula, Paszke, Vander{P}las, Wanderman-{M}ilne, and
  Zhang]{jax2018github}
James Bradbury, Roy Frostig, Peter Hawkins, Matthew~James Johnson, Chris Leary,
  Dougal Maclaurin, George Necula, Adam Paszke, Jake Vander{P}las, Skye
  Wanderman-{M}ilne, and Qiao Zhang.
\newblock {JAX}: composable transformations of {P}ython+{N}um{P}y programs,
  2018.
\newblock URL \url{http://github.com/google/jax}.

\bibitem[Wilson and Nickisch(2015)]{wilson2015kernel}
Andrew Wilson and Hannes Nickisch.
\newblock Kernel interpolation for scalable structured {Gaussian} processes
  ({KISS-GP}).
\newblock In \emph{International Conference on Machine Learning}, pages
  1775--1784. PMLR, 2015.

\bibitem[Harbrecht et~al.(2012)Harbrecht, Peters, and
  Schneider]{harbrecht2012low}
Helmut Harbrecht, Michael Peters, and Reinhold Schneider.
\newblock On the low-rank approximation by the pivoted {Cholesky}
  decomposition.
\newblock \emph{Applied Numerical Mathematics}, 62\penalty0 (4):\penalty0
  428--440, 2012.

\bibitem[Charlier et~al.(2021)Charlier, Feydy, Glaunes, Collin, and
  Durif]{charlier2021kernel}
Benjamin Charlier, Jean Feydy, Joan~Alexis Glaunes, Fran{\c{c}}ois-David
  Collin, and Ghislain Durif.
\newblock Kernel operations on the {GPU}, with autodiff, without memory
  overflows.
\newblock \emph{Journal of Machine Learning Research}, 22\penalty0
  (74):\penalty0 1--6, 2021.

\bibitem[Kingma and Ba(2015)]{KingBa15}
Diederik Kingma and Jimmy Ba.
\newblock Adam: A method for stochastic optimization.
\newblock San Diega, CA, USA, 2015.

\bibitem[Dormand and Prince(1980)]{dormand1980family}
J.~R. Dormand and P.~J. Prince.
\newblock A family of embedded {R}unge--{K}utta formulae.
\newblock \emph{J. Comp. Appl. Math}, 6:\penalty0 19--26, 1980.

\bibitem[Shampine(1986)]{shampine1986some}
Lawrence~F. Shampine.
\newblock Some practical {R}unge-{K}utta formulas.
\newblock \emph{Mathematics of Computation}, 46\penalty0 (173):\penalty0
  135--150, 1986.
\newblock \doi{https://doi.org/10.2307/2008219}.

\bibitem[Chen et~al.(2018)Chen, Rubanova, Bettencourt, and
  Duvenaud]{chen2018neural}
Ricky~TQ Chen, Yulia Rubanova, Jesse Bettencourt, and David~K Duvenaud.
\newblock Neural ordinary differential equations.
\newblock \emph{Advances in Neural Information Processing Systems}, 31, 2018.

\bibitem[Tsitouras(2011)]{tsitouras2011runge}
Charalampos Tsitouras.
\newblock {Runge--Kutta} pairs of order 5 (4) satisfying only the first column
  simplifying assumption.
\newblock \emph{Computers \& Mathematics with Applications}, 62\penalty0
  (2):\penalty0 770--775, 2011.

\bibitem[Guo et~al.(2023)Guo, Lu, Liu, Cheng, and Hu]{guo2023visual}
Meng-Hao Guo, Cheng-Ze Lu, Zheng-Ning Liu, Ming-Ming Cheng, and Shi-Min Hu.
\newblock Visual attention network.
\newblock \emph{Computational Visual Media}, 9\penalty0 (4):\penalty0 733--752,
  2023.

\bibitem[Deng et~al.(2009)Deng, Dong, Socher, Li, Li, and
  Fei-Fei]{deng2009imagenet}
Jia Deng, Wei Dong, Richard Socher, Li-Jia Li, Kai Li, and Li~Fei-Fei.
\newblock {ImageNet}: A large-scale hierarchical image database.
\newblock In \emph{IEEE Conference on Computer Vision and Pattern Recognition},
  pages 248--255. IEEE, 2009.

\bibitem[Deng et~al.(2022)Deng, Zhou, and Zhu]{deng2022accelerated}
Zhijie Deng, Feng Zhou, and Jun Zhu.
\newblock Accelerated linearized {Laplace} approximation for {Bayesian} deep
  learning.
\newblock \emph{Advances in Neural Information Processing Systems},
  35:\penalty0 2695--2708, 2022.

\bibitem[Zhou et~al.(2014)Zhou, Lapedriza, Xiao, Torralba, and
  Oliva]{zhou2014learning}
Bolei Zhou, Agata Lapedriza, Jianxiong Xiao, Antonio Torralba, and Aude Oliva.
\newblock Learning deep features for scene recognition using places database.
\newblock \emph{Advances in Neural Information Processing Systems}, 27, 2014.

\bibitem[Liu et~al.(2023)Liu, Benitez, Faucher, Khorashadizadeh, de~Hoop, and
  Dokmani{\'c}]{liu2023wavebench}
Tianlin Liu, Jose Antonio~Lara Benitez, Florian Faucher, AmirEhsan
  Khorashadizadeh, Maarten~V de~Hoop, and Ivan Dokmani{\'c}.
\newblock {WaveBench}: Benchmarks datasets for modeling linear wave propagation
  {PDEs}.
\newblock \emph{Transactions on Machine Learning Research}, 2023.

\bibitem[Deuflhard et~al.(2003)Deuflhard, Hohmann, Deuflhard, and
  Hohmann]{deuflhard2003three}
Peter Deuflhard, Andreas Hohmann, Peter Deuflhard, and Andreas Hohmann.
\newblock Three-term recurrence relations.
\newblock \emph{Numerical Analysis in Modern Scientific Computing: An
  Introduction}, pages 151--178, 2003.

\bibitem[Deuflhard(1976)]{deuflhard1976algorithms}
Peter Deuflhard.
\newblock On algorithms for the summation of certain special functions.
\newblock \emph{Computing}, 17:\penalty0 37--48, 1976.

\bibitem[Bartels et~al.(2023)Bartels, Stensbo-Smidt, Moreno-Mu{\~n}oz, Boomsma,
  Frellsen, and Hauberg]{bartels2023adaptive}
Simon Bartels, Kristoffer Stensbo-Smidt, Pablo Moreno-Mu{\~n}oz, Wouter
  Boomsma, Jes Frellsen, and Soren Hauberg.
\newblock Adaptive {Cholesky} {Gaussian} processes.
\newblock In \emph{International Conference on Artificial Intelligence and
  Statistics}, pages 408--452. PMLR, 2023.

\bibitem[Camachol(1998)]{camachol1998inducing}
Rui Camachol.
\newblock Inducing models of human control skills.
\newblock In \emph{Machine Learning: ECML-98: 10th European Conference on
  Machine Learning Chemnitz, Germany, April 21--23, 1998 Proceedings 10}, pages
  107--118. Springer, 1998.

\bibitem[Schwaighofer and Tresp(2002)]{schwaighofer2002transductive}
Anton Schwaighofer and Volker Tresp.
\newblock Transductive and inductive methods for approximate {Gaussian} process
  regression.
\newblock \emph{Advances in Neural Information Processing Systems}, 15, 2002.

\bibitem[Shannon et~al.(2003)Shannon, Markiel, Ozier, Baliga, Wang, Ramage,
  Amin, Schwikowski, and Ideker]{shannon2003cytoscape}
Paul Shannon, Andrew Markiel, Owen Ozier, Nitin~S Baliga, Jonathan~T Wang,
  Daniel Ramage, Nada Amin, Benno Schwikowski, and Trey Ideker.
\newblock Cytoscape: a software environment for integrated models of
  biomolecular interaction networks.
\newblock \emph{Genome research}, 13\penalty0 (11):\penalty0 2498--2504, 2003.

\bibitem[Prince and Dormand(1981)]{prince1981high}
P.~J Prince and J.~R. Dormand.
\newblock High order embedded {R}unge--{K}utta formulae.
\newblock \emph{Journal of Computational and Applied Mathematics}, 7\penalty0
  (1):\penalty0 67--75, 1981.

\bibitem[Stanziola et~al.(2023)Stanziola, Arridge, Cox, and
  Treeby]{stanziola2023j}
Antonio Stanziola, Simon~R Arridge, Ben~T Cox, and Bradley~E Treeby.
\newblock j-{Wave}: An open-source differentiable wave simulator.
\newblock \emph{SoftwareX}, 22:\penalty0 101338, 2023.

\bibitem[MacKay(1992)]{mackay1992bayesian}
David~JC MacKay.
\newblock {Bayesian} interpolation.
\newblock \emph{Neural computation}, 4\penalty0 (3):\penalty0 415--447, 1992.

\bibitem[Immer et~al.(2021{\natexlab{b}})Immer, Korzepa, and
  Bauer]{immer2021improving}
Alexander Immer, Maciej Korzepa, and Matthias Bauer.
\newblock Improving predictions of {Bayesian} neural nets via local
  linearization.
\newblock In \emph{International Conference on Artificial Intelligence and
  Statistics}, pages 703--711. PMLR, 2021{\natexlab{b}}.

\bibitem[Papyan(2020)]{papyan2020traces}
Vardan Papyan.
\newblock Traces of class/cross-class structure pervade deep learning spectra.
\newblock \emph{Journal of Machine Learning Research}, 21\penalty0
  (252):\penalty0 1--64, 2020.

\bibitem[Bekas et~al.(2007)Bekas, Kokiopoulou, and Saad]{bekas2007estimator}
Costas Bekas, Effrosyni Kokiopoulou, and Yousef Saad.
\newblock An estimator for the diagonal of a matrix.
\newblock \emph{Applied Numerical Mathematics}, 57\penalty0 (11-12):\penalty0
  1214--1229, 2007.

\bibitem[Naeini et~al.(2015)Naeini, Cooper, and
  Hauskrecht]{naeini2015obtaining}
Mahdi~Pakdaman Naeini, Gregory Cooper, and Milos Hauskrecht.
\newblock Obtaining well calibrated probabilities using {Bayesian} binning.
\newblock In \emph{Proceedings of the AAAI Conference on Artificial
  Intelligence}, volume~29, 2015.

\end{thebibliography}

\newpage

\appendix
\section*{Appendix: Overview}

Some of the results in the main paper promised detailed information about setups, data, compute, or additional proofs.
For example, the case study about partial differential equations involves a data generation process which will receive further explanation in this supplement.

The appendix provides the following details:
\Cref{appendix-section-mode-vjp-wall-times,appendix-section-accuracy-loss-hilbert} elaborate on \Cref{figure-wall-times-per-vjp,table-accuracy-loss-hilbert} respectively; 
\Cref{appendix-section-proof-of-corollary-parameter-gradients-arnoldi,appendix-section-proof-of-theorem-adjoint-system-of-lanczos,appendix-section-proof-of-theorem-adjoint-system-of-arnoldi,appendix-section-solving-the-adjoint-system-in-detail}
contain proofs for the main results; and \Cref{appendix-section-pde-data,appendix-section-low-memory-matrix-vector-products,appendix-experiment-details-gaussian-processes,appendix-laplace} describe the setup used for the case studies.
Notably, \Cref{appendix-section-low-memory-matrix-vector-products} describes how we implement low-memory matrix-vector products in JAX (to replicate what makes libraries like KeOps \citep{charlier2021kernel} so efficient), and \Cref{appendix-section-pde-data} outlines a PDE data set similar to that by \citet{liu2023wavebench}.

\paragraph{Code}
Most of the contributions of this paper pertain to differentiable implementations of numerical algorithms -- at the heart of it are our reverse-mode differentiable Lanczos and Arnoldi iterations.
We provide JAX code to reproduce all experiments at the URL
\begin{center}
\texttt{https://github.com/pnkraemer/experiments-lanczos-adjoints}
\end{center}
and have packaged all numerical methods in a JAX library that can be installed via
\begin{center}
\texttt{pip install matfree}
\end{center}
Next to Lanczos and Arnoldi, this includes variants of conjugate gradient methods \citep{hestenes1952methods} and pivoted Cholesky preconditioners \citep{harbrecht2012low} (which we used for Gaussian processes), low-memory kernel-matrix-vector products (also Gaussian processes), efficient GGN-vector products (used for Bayesian neural networks), and matrix-free sampling algorithms to sample from Gaussian processes and Laplace-approximated Bayesian neural networks (used for Gaussian processes and Bayesian neural networks).
These methods are known in some form or another, but until now, they have all lacked a software implementation in the current JAX ecosystem (with some exceptions relating to conjugate gradients).

\paragraph{Compute}
All experiments before the case studies were run on CPU. The Gaussian process and differential equation case studies run on a V100 GPU, the Bayesian neural network one on a P100 GPU. The Gaussian process and Bayesian neural network studies run in a few hours, all other code finishes in a few minutes.

\section{Additional context for \Cref{figure-wall-times-per-vjp}}
\label{appendix-section-mode-vjp-wall-times}

To create \Cref{figure-wall-times-per-vjp}, we load the ``bcsstk18'' matrix from the SuiteSparse matrix collection \citep{kolodziej2019suitesparse,davis2011university,duff1989sparse}.
This matrix is symmetric and has 11,948 rows/columns and 149,090 nonzero entries.
We implement matrix-vector products with \texttt{jax.experimental.sparse}, and use a Lanczos iteration without reorthogonalisation.
We time the execution of the forward pass, as well as the backward pass with and without implementing a custom vector-Jacobian product (the custom vector-Jacobian product involves the adjoint).
The results were shown in \Cref{figure-wall-times-per-vjp}, and displayed how rapidly the computational complexity of automatic differentiation increases, whereas the computational complexity of our custom gradient mirrors that of the forward pass.
This figure showed how without our proposed gradients, differentiating through the Lanczos iteration is unfeasible.

To add to this benchmark, we repeat the same for the Arnoldi iteration and show both compilation and runtime for Lanczos and Arnoldi in \Cref{figure-runtime-vjp-all} (this includes the curves from \Cref{figure-wall-times-per-vjp} again).
\begin{figure}[ht]
\caption{Run and compilation times for Lanczos and Arnoldi. The ``backprop'' curves were stopped at $100$, because for higher values we encountered memory issues. All experiments run on CPU.}
\label{figure-runtime-vjp-all}
\includegraphics[width=\linewidth]{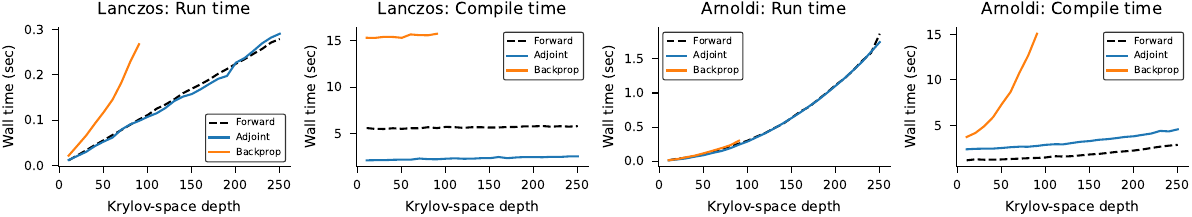}
\end{figure}
We see different behaviour for Lanczos and Arnoldi.
Whereas for back-propagation through Lanczos without the custom gradient, the compilation times remain constant for increasing Krylov-space depth $K$ and runtimes increase rapidly, the reverse is true for the Arnoldi iteration.
The adjoint method mirrors that of the forward pass in both benchmarks.
In either case, the increasing memory requirements for backpropagation through Lanczos and Arnoldi without our proposed adjoints becomes apparent.

\section{Proof of \cref{theorem-adjoint-system-of-arnoldi}}
\label{appendix-section-proof-of-theorem-adjoint-system-of-arnoldi}

Here is how we derive the adjoints of the Arnoldi system.
The structure is the usual: take the forward constraint, differentiate, add Lagrange multipliers (``transpose''), and identify the adjoint system.

\subsection{Linearisation}
The linearisation of \Cref{equation-arnoldi-constraint} is
\begin{subequations}
\begin{align}
(\diff A) Q + A \diff Q - (\diff Q) H - Q \diff H - (\diff r) (e_K)^\top
	&= 0 
		\quad \in \Rbb^{N \times K}, 
\\
(\diff Q) e_1 - (\diff v)c - v \diff c
	&= 0 
		\quad \in \Rbb^{N \times 1},
\\
I_\leq \circ [ \langle Q e_i, (\diff Q) e_j \rangle  + \langle Q e_j, (\diff Q) e_i \rangle ]_{i,j=1}^K 
	&= 0 
		\quad \in \Rbb^{K\times K}, \\
I_\ll \circ \diff H
	&= 0 
		\quad \in \Rbb^{K\times K},
\\
[\langle r, (\diff Q) e_j \rangle + \langle Q e_j, \diff r\rangle]_{j=1}^K
	&= 0 
		\quad \in \Rbb^{K \times 1}.
\end{align}
\end{subequations}
To see this, apply the chain- and product rules to the original constraint in \Cref{equation-arnoldi-constraint}.

\subsection{Transposition}

Let $\rho = \rho(Q, r, H, c) \in \Rbb$ be a scalar function of the outputs.
In the following, interpret vectors as $N \times 1$ matrices and scalars as $1 \times 1$ matrices.
The values of inner products and the realisations of $\rho$ are the only scalars.

For all $\Lambda \in \Rbb^{N \times K}$, 
$\lambda \in \Rbb^{N \times 1}$,
$\Gamma, \Sigma \in \Rbb^{K \times K}$,
and
$\gamma \in \Rbb^{K \times 1}$, we have
\begin{subequations}
\begin{align}
\diff \rho 
&= 
\langle
	\nabla_Q \rho,
	\diff Q
\rangle
+
\langle
	\nabla_H \rho,
	\diff H
\rangle
+
\langle
	\nabla_r \rho,
	\diff r
\rangle
+
\langle
	\nabla_c \rho,
	\diff c 
\rangle
\\
&= 
\langle
	\nabla_Q \rho,
	\diff Q
\rangle
+
\langle
	\nabla_H \rho,
	\diff H
\rangle
+
\langle
	\nabla_r \rho,
	\diff r
\rangle
+
\langle
	\nabla_c \rho,
	\diff c 
\rangle
\notag
\\
&~~~~~~~~~~+
\langle
	\Lambda,
	(\diff A) Q + A \diff Q - (\diff Q) H - Q \diff H - (\diff r) (e_K)^\top
\rangle
\notag
\\
&~~~~~~~~~~+
\langle
	\lambda,
	(\diff Q) e_1 - (\diff v) c - v \diff c
\rangle
\notag
\\
&~~~~~~~~~~+
\langle
	\Gamma,
	I_\leq \circ [\langle Q e_i, (\diff Q) e_j \rangle + \langle Q e_j, (\diff Q) e_i \rangle]_{i,j=1}^K
\rangle
\notag
\\
&~~~~~~~~~~+
\langle
	\Sigma,
	I_\ll \circ \diff H
\rangle
\notag
\\
&~~~~~~~~~~+
\langle
	\gamma,
	[\langle r, (\diff Q) e_j \rangle + \langle Q e_j, \diff r \rangle]_{j=1}^K
\rangle
\\
&= 
\langle
	\nabla_Q \rho,
	\diff Q
\rangle
+
\langle
	\nabla_H \rho,
	\diff H
\rangle
+
\langle
	\nabla_r \rho,
	\diff r
\rangle
+
\langle
	\nabla_c \rho,
	\diff c 
\rangle
\notag
\\
&~~~~~~~~~~+
\langle
	\Lambda Q^\top,
	\diff A
\rangle 
+ 
\langle 
	A^\top\Lambda,
	\diff Q 
\rangle
-
\langle
	\Lambda H^\top,
	\diff Q
\rangle
-
\langle
	 Q^\top\Lambda,
	\diff H
\rangle
- 
\langle
	\Lambda e_K,
	\diff r
\rangle
\notag
\\
&~~~~~~~~~~+
\langle
	\lambda (e_1)^\top,
	\diff Q 
\rangle
- 
\langle
	\lambda c^\top,
	\diff v
\rangle
- 
\langle
	v^\top\lambda ,
	\diff c
\rangle
\notag
\\
&~~~~~~~~~~+
\langle
	Q(I_\leq \circ \Gamma),
	\diff Q
\rangle
+
\langle
	Q (I_\leq \circ \Gamma)^\top,
	\diff Q
\rangle
\notag
\\
&~~~~~~~~~~+
\langle
	I_\ll \circ \Sigma,
	\diff H
\rangle
\notag
\\
&~~~~~~~~~~+
\langle
	r \gamma^\top,
	\diff Q
\rangle
+
\langle 
	Q \gamma,
	\diff r
\rangle
\\
&\eqqcolon
\langle 
	Z_Q, 
	\diff Q 
\rangle
+
\langle
	Z_H,
	\diff H
\rangle
+
\langle
	Z_r,
	\diff r
\rangle
+
\langle
	Z_c,
	\diff c
\rangle
+
\langle
	\Lambda Q^\top,
	\diff A
\rangle
+
\langle
	\lambda c^\top,
	\diff v
\rangle
\end{align}
\end{subequations}
with the constraints
\begin{subequations}
\begin{align}
Z_Q
&\coloneqq
\nabla_Q \rho + A^\top \Lambda - \Lambda H^\top + \lambda (e_1)^\top + Q(I_\leq \circ \Gamma) + Q(I_\leq \circ \Gamma)^\top + r \gamma^\top
&&\in \Rbb^{N \times K}
\\
Z_H
&\coloneqq
\nabla_H \rho - Q^\top\Lambda  + I_\ll \circ \Sigma
&&\in \Rbb^{K \times K}
\\
Z_r
&\coloneqq
\nabla_r \rho - \Lambda e_K + Q \gamma
&&\in \Rbb^{N \times 1}
\\
Z_c
&\coloneqq
\nabla_c \rho - v^\top\lambda 
&&\in \Rbb^{1 \times 1}.
\end{align}
\end{subequations}
Solving the adjoint system, $Z_Q = 0$, $Z_H = 0$, $Z_r = 0$, and $Z_c = 0$ as a function of $\Lambda, \lambda, \Gamma, \Sigma$, and $\gamma$, yields the desired $\nabla_A \rho = \Lambda Q^\top$ and $\nabla_v f = \lambda c^\top$.
\Cref{theorem-adjoint-system-of-arnoldi} is complete.

\section{Proof of \cref{theorem-adjoint-system-of-lanczos}}
\label{appendix-section-proof-of-theorem-adjoint-system-of-lanczos}

Recall that  $\rho = \rho(x_1, ..., x_{K+1}; a_1, ..., a_K; b_1, ..., b_K)$ shall be a scalar/loss that depends on the output of the algorithm. 
Denote by $\nabla_{x_k} \rho$ the gradient of $\rho$ with respect to each Lanczos vector $x_k$, and by $\nabla_{a_k} \rho$ and $\nabla_{b_k} \rho$ the gradients with respect to $a_k$ and $b_k$ respectively.

The differential of normalisation, i.e., the operation $s \mapsto h = s/(s^\top s)$ is
\begin{align}
\diff h = \frac{1}{s^\top s} \left( I - h h^\top \right) \diff s.
\end{align}
The next steps are the usual ones: we start with the forward constraint, linearise, add Lagrange multipliers, and identify the adjoint system.
The order of the middle two steps (linearise, multipliers) is interchangeable; while for Arnoldi, we linearise first and then add Lagrange multipliers, for Lanczos, we go the other way.

Define Lagrange multipliers $\{\lambda_k\}_{k=0}^{K} \subseteq \Rbb^n$ and $\{\mu_{k}/2, \nu_k\}_{j,k=1}^{K} \subseteq \Rbb$, %
\begin{align}
\rho 
	&= 
		\rho
		+ \sum_{k=1}^K \lambda_{k}^\top \left(-b_{k-1}x_{k-1} + (A - a_k I)x_k - b_k x_{k+1} \right) 
		\tag{$b_0 = 1, x_0 = 0$} 
		\\
	&~~~~~~~
		- \lambda_0^\top \left(x_1 - \frac{v}{\sqrt{v^\top v}} \right) 
		+ \frac{1}{2}\sum_{k=1}^{K} \mu_{k}(x_{k+1}^\top x_{k+1} - 1)
		+ \sum_{k=1}^{K} \nu_{k}\,x_{k}^\top x_{k+1}.
\end{align}
Differentiate and use that the forward constraint must be satisfied,
\begin{align}
\diff \rho 
	&=
		 \sum_{k=1}^{K+1} (\nabla_{x_k} \rho)^\top \diff x_{k} 
		+ \sum_{k=1}^K (\nabla_{a_k} \rho)^\top \diff a_k 
		+ \sum_{k=1}^K (\nabla_{b_k} \rho)^\top \diff b_k
		\notag
		\\
	&~~~
		+ \sum_{k=1}^K \lambda_{k}^\top 
		\left[ 
			(\diff A) x_k 
			- (\diff a_k)x_k 
			- (\diff b_{k-1}) x_{k-1} 
			- (\diff b_k) x_{k+1}
		\right]
		\tag{$\diff b_0 = 1, \diff x_0 = 0$} 
		\\
	&~~~
		+ \sum_{k=1}^K \lambda_{k}^\top 
		\left[ 
			A \diff x_k 
			- a_k \diff x_k 
			- b_{k-1}\diff x_{k-1} 
			- b_k \diff x_{k+1}
		\right] 
		\tag{$b_0 = 1, x_0 = 0$} 
		\\
	&~~~
		-\lambda_0^\top 
		(
			\diff x_1 
			- (I - x_1 x_1^\top) /(v^\top v)\diff v
		)
		\notag
		\\
	&~~~
		+\sum_{k=1}^{K}\mu_{k}\,x_{k+1}^\top \diff x_{k+1}
		\notag
		\\
	&~~~
		+\sum_{k=1}^{K} \nu_{k}(
			x_{k+1}^\top \diff x_{k} 
			+ x_{k}^\top \diff x_{k+1}
		).
\end{align}
Sort all terms by differential,
\begin{align}
\diff \rho 
	&=
	\sum_{k=1}^K Z_{a_k} \diff a_k 
	+ \sum_{k=1}^K Z_{b_k} \diff b_k 
	+ \sum_{k=1}^{K+1} Z_{x_k} \diff x_k
	+ Z_v\diff v
	+ \trace{Z_A \diff A }.
\end{align}
so that enforcing that all $Z_{a_k}$, $Z_{b_k}$, $Z_{x_k}$ terms are zero yields constraints for the multipliers from which we can compute gradients with respect to $v$ and $A$.

What are those terms? Let $\lambda_{K+1} = 0$, $\mu_0=0$, and $\nu_0=0$ (to simplify notation below); then,
\begin{align}
Z_{x_{K+1}} 
	&= 
		- \lambda_{K} b_K 
		+ (
			\nabla_{x_{K+1}} \rho 
			+  \mu_K x_{K+1} 
			+ \nu_{K} x_K
		),  
\end{align}
and for all  $k=K, ...1$, {\hfill ({recall}~$x_0=0, b_0=1, \mu_0=0, \nu_0=0, \lambda_{K+1}=0$)}
\begin{subequations}
\begin{align}
Z_{x_k} 
	&= 
		- b_k \lambda_{k+1} 
		+ (A^\top - a_k I)\lambda_{k} 
		- b_{k-1} \lambda_{k-1} 
		+ (
			\nabla_{x_k} \rho 
			+ \mu_{k-1} x_{k} 
			+ \nu_k x_{k+1}
			+ \nu_{k-1} x_{k-1}
		),  
	\\
Z_{a_k} 
	&= 
		\nabla_{a_k} \rho 
		- \lambda_{k}^\top x_k,
		\\
Z_{b_k} 
	&= 
		\nabla_{b_k} \rho 
		- \lambda_{k+1}^\top x_k 
		- \lambda_{k}^\top x_{k+1}.
\end{align}
\end{subequations}
The expressions are like the forward-Lanczos constraints, and the main differences are (i) that the recursions are run backwards in ``time'' and (ii) the existence of a nonzero bias term in the adjoints (marked by parentheses).
Enforcing $Z_{a_k}$, $Z_{b_k}$, $Z_{x_k}$ to be zero identifies 
\begin{align}
\nabla_v \rho = Z_v \coloneqq \left( \frac{\lambda_0^\top x_1}{v^\top v} x_1^\top- \lambda_0^\top \right) 
,
\quad
\nabla_A \rho = Z_A \coloneqq\left[\sum_{k=1}^K x_k \lambda_{k}^\top\right].
\end{align}
\Cref{theorem-adjoint-system-of-lanczos} is complete.

\section{Proof of \Cref{corollary-parameter-gradients-arnoldi}}
\label{appendix-section-proof-of-corollary-parameter-gradients-arnoldi}

The following derivation covers only the case for Lanczos, i.e., we use the variables $x_1, ..., x_{K+1}$ instead of Arnoldi's $q_1, ..., q_K$.
But the derivation is the same for both methods.

The expression we manipulate is
\begin{align}
\diff \rho = \trace{\left[\sum_{k=1}^K x_k \lambda_{k}^\top\right] \diff A} + \text{const}
\end{align}
where all non-$\diff A$-related quantities are treated as some unimportant constants.

For a single parameter $\theta_j$, $\diff A = (\nabla_{\theta_j} A)^\top \diff \theta_j$ and we have
\begin{align}
\trace{\left[\sum_{k=1}^K x_k \lambda_{k}^\top\right] D_j A \diff \theta_j} 
	&= \trace{\left[\sum_{k=1}^K x_k \lambda_{k}^\top\right] (\nabla_{\theta_j} A)^\top }\diff \theta_j  \tag{chain rule}
	\\
	&= \trace{\sum_{k=1}^K \lambda_{k}^\top (\nabla_{\theta_j} A)^\top x_k} \diff \theta_j \tag{cyclic property of traces}
	\\
	&= \sum_{k=1}^K \lambda_{k}^\top(\nabla_{\theta_j} A)^\top x_k \diff \theta_j \tag{a scalar is its own trace}
	\\
	&= \nabla_{\theta_j} \left[\sum_{k=1}^K \lambda_{k}^\top A^\top x_k\right] \diff \theta_j \tag{linearity of diff. \& summation} \\
	&= \nabla_{\theta_j} \left[\sum_{k=1}^K \lambda_{k}^\top A x_k\right] \diff \theta_j. \tag{symmetry of $A$}
\end{align}
In conclusion, the derivative of $\rho$ wrt $\theta_j$ is
\begin{align}
\nabla_{\theta_j} \rho =  \nabla_j \left[\sum_{k=1}^K  x_k^\top A^\top \lambda_{k}\right]
\end{align}
and stacking all of those partial derivatives on top of each other, we obtain
\begin{align}
\nabla_{\theta} \rho 
=
[\nabla_{\theta_j} \rho]_j
=
\nabla \left[ \theta \mapsto  \sum_{k=1}^K  x_k^\top A(\theta)^\top \lambda_{k}\right] 
=
\sum_{k=1}^K \nabla \left[ \theta \mapsto    A(\theta)^\top \lambda_{k}\right]  x_k.
\end{align}
We already compute $A(\theta)^\top \lambda_{k}$ during the backward pass, so we are a single vector-Jacobian product with $x_k$ away from a matrix-parameter-gradient instead of a matrix-gradient.
This requires $O(p+n)$ storage, and is computed online, which makes the memory-complexity independent of $K$.

\section{Solving the adjoint system}
\label{appendix-section-solving-the-adjoint-system-in-detail}

The upcoming section details how to solve the adjoint system for both, the Lanczos and the Arnoldi iterations.
It reuses notation from \Cref{appendix-section-proof-of-theorem-adjoint-system-of-lanczos,appendix-section-proof-of-theorem-adjoint-system-of-arnoldi}.

We begin with Lanczos, because the solution is less technical, and because starting with Lanczos can provide a template for solving Arnoldi's adjoint system.
All results in the present section (except for those that explicitly point to \citet{deuflhard2003three,deuflhard1976algorithms}, which are marked as such) are new and a contribution of this work.

\subsection{Lanczos}

The inputs to the adjoint system are the Lanczos vectors $\{x_k\}_{k=1}^{K+1}$ and the coefficients $\{a_k\}_{k=1}^K$ as well as $\{a_k\}_{k=1}^K$ from the forward pass, the corresponding input derivatives $\{\nabla_{x_k} \rho\}_{k=1}^{K+1}$, $\{\nabla_{a_k} \rho\}_{k=1}^K$, and $\{\nabla_{a_k} \rho\}_{k=1}^K$, and matrix $A$ and initial vector $v$.

The overall strategy for solving the adjoint system of the Lanczos iteration (\Cref{theorem-adjoint-system-of-lanczos}) is the following: for every $k=K, ..., 1$, alternate the two steps:
\begin{enumerate}
	\item Combine orthogonality with the $Z_{a_k}$ constraints to get $\nu_k$, and combine it with the $Z_{b_k}$ constraints to get $\mu_k$.
	\item Once each $\mu_k$ and $\nu_k$ are available, solve for $\lambda_k$ and repeat with the next lower $k$.
\end{enumerate}
This results in the following procedure: 
To start, set $\zeta_{K+1} = -(\nabla_{x_{K+1}} \rho)$ and $\lambda_{K+1} = 0$.
Then, for all $k=K, ..., 1$, compute
\begin{subequations}
\begin{align}
\xi_{k} &= \zeta_{k+1} / b_k \\
\tilde{\mu}_k &= \nabla_{b_k} \rho - \lambda_{k+1}^\top x_k + x_{k+1}^\top \xi_{k} \\
\tilde{\nu}_k &= \nabla_{a_k} \rho + x_k^\top \xi_{k} \\
\lambda_k &= -\xi_k + \tilde{\mu}_k \cdot x_{k+1} + \tilde{\nu}_k \cdot x_k \\
\zeta_k &= -\nabla_{x_k} \rho - A^\top \lambda_k + a_k \cdot \lambda_k + b_k \cdot \lambda_{k+1} - b_k \cdot \tilde{\nu}_k \cdot x_{k+1} \\
\text{Repeat with } k &= k-1.
\end{align}
\end{subequations}
Finally, set $\lambda_0 = \zeta_1$.
Only $\lambda$ and $\zeta$ affect subsequent steps; $\mu$, $\nu$, and $\xi$ are only needed for computing $\lambda$ and $\zeta$.
The $k$-th step depends on $a_k, b_k, x_k, x_{k+1}, \nabla_{a_k} \rho, \nabla_{b_k} \rho, \nabla_{x_k} \rho$.

The strategy above yields all $\{\lambda_k\}_{k=0}^K$.
Finalise the gradients
\begin{align}
\nabla_v \rho = \frac{\lambda_0^\top x_1}{v^\top v} x_1- \lambda_0, 
\quad 
\nabla_A \rho = \sum_{k=1}^K \lambda_{k} x_k^\top
\end{align}
which can be embedded into any reverse-mode algorithmic-differentiation-engine.

\subsection{Arnoldi}

The process for Arnoldi is similar to that for Lanczos, but the derivation is more technical.
It shares many similarities with deriving the Arnoldi iteration (i.e., the forward pass), so we begin by providing a perspective on recurrence relations (which include the Arnoldi iteration) through the lens of linear system solvers \citep{deuflhard2003three,deuflhard1976algorithms} before we use this perspective to solve the adjoint system.

\subsubsection{Solving the original system}

At its core, finding the Arnoldi vectors amounts to solving
\begin{subequations}
\begin{align}
- AQ + QH + r (e_K)^\top &= 0,
\\
Q e_1 -  c v &= 0,
\end{align}
\end{subequations}
which is possible in closed form as follows:
We rewrite the first two constraints as 
\begin{subequations}
\begin{align}
(e_1 \otimes I) \vect{Q} &= \vect{cv}
\\
- (I \otimes A) \vect{Q} + (H^\top \otimes I) \vect{Q} + (e_K \otimes I) \vect{r} &= 0.
\end{align}
\end{subequations}
This expression is equivalent to
\begin{align}\label{equation-arnoldi-as-linear-system}
\begin{pmatrix}
e_1 \otimes I & 0 \\
H^\top \otimes I - I \otimes A & e_K \otimes I
\end{pmatrix}
\begin{pmatrix}
\vect{Q} \\
\vect{r}
\end{pmatrix}
=
\begin{pmatrix}
\vect{cv} \\
0
\end{pmatrix}
\in \Rbb^{N(K+1) \times 1}
\end{align}
This is similar to the work by \citet{deuflhard1976algorithms}, who explain adjoints of three-term recurrence relations.
Since $c^2 = \langle v, v \rangle$ holds (ie, $c$ is known), the first row of $Q$ is known.
Then, the first row of $Q$ together with the orthogonality constraints yields the first row of $H^\top$, which then defines the next row of the linear system thus the second row of $Q$.
Alternating between deriving the next row of $H^\top$ and solving the lower triangular systems is then Arnoldi's algorithm:
\begin{algorithm}[Arnoldi's forward pass; paraphrased]
Assume that $v$ and $K$ are known. Compute $c = \sqrt{\langle v, v \rangle}$.
Then, for $k=1, ..., K$, alternate the following tsteps:
\begin{enumerate}
	\item Derive the next column of $H$ using the orthogonality constraints.
	\item Forward-substitute (``solve'') the block-lower-triangular system for the next column of $Q$ (respectively $r$ at the last iteration).
\end{enumerate}
Return all $Q$ and $H$, as well as $c$ and $r$.
\end{algorithm}
The same principle applies to the adjoint system, and the only difference is that the notation is slightly more complicated:

\subsubsection{Solving the adjoint system}
The constraints $Z_Q=0$ and $Z_r = 0$ mirror those of $AQ - QH - r (e_K)^\top = 0$ and $Q e_1 - c v = 0$; the constraints $Z_H = 0$ and $Z_c = 0$ mirror the orthogonality constraints $Q^\top Q = I$ and $Q^\top r = 0$.
Therefore, we start with $Z_Q = 0$ and $Z_r = 0$,
\begin{subequations}
\begin{align}
\nabla_Q f + A^\top \Lambda - \Lambda H^\top + \lambda (e_1)^\top + Q(I_\leq \circ \Gamma) + Q(I_\leq \circ \Gamma)^\top + r \gamma^\top &= 0
\\
\nabla_r f - \Lambda e_K + Q \gamma &= 0.
\end{align}
\end{subequations}
Introduce the auxiliary quantities
\begin{subequations}
\begin{align}
\Psi(\Gamma, \gamma) &\coloneqq \nabla_Q f + Q(I_\leq \circ \Gamma) + Q(I_\leq \circ \Gamma)^\top + r \gamma^\top &&\in \Rbb^{N \times K}
\\
\psi(\gamma) &\coloneqq \nabla_r f + Q \gamma &&\in \Rbb^{N \times 1}.
\end{align}
\end{subequations}
$\Psi$ and $\psi$ only serve the purpose of simplifying the notation in the coming part; there is no ``meaning'' associated with them.
Vectorise both expressions,
\begin{subequations}
\begin{align}
 (I \otimes A^\top) \vect{\Lambda} - [(H^\top)^\top \otimes I] \vect{\Lambda} + (e_1 \otimes I) \vect{\lambda} &= -\vect{\Psi(\Gamma, \gamma)} \\
 [(e_K)^\top \otimes I] \vect{\Lambda} &= \vect{\psi(\gamma)}
\end{align}
\end{subequations}
and observe that this can be written as a linear system 
\begin{align}
\begin{pmatrix}
e_1 \otimes I & I \otimes A^\top - (H^\top)^\top \otimes I \\
0 & (e_K)^\top \otimes I 
\end{pmatrix}
\begin{pmatrix}
\vect{\lambda} \\
\vect{\Lambda}
\end{pmatrix}
=
\begin{pmatrix}
-\vect{\Psi(\Gamma, \gamma)} \\
\vect{\psi(\gamma)}
\end{pmatrix}
\end{align}
with a system matrix that is the transpose of the system matrix of the forward pass.
The matrix is upper triangular, and the equation can be solved with backward substitution provided $\psi(\gamma)$ and $\Psi(\Gamma, \gamma)$ are known.

The defining quantities $\Psi$ and $\psi$ emerge by combining the adjoint recursion with the projection constraints $Z_H = 0$ (for $\Gamma$, which yields $\Psi$) and $Z_r = 0$ (for $\gamma$, which yields $\psi$). We use $Z_c = 0$ to get a single element in $\Gamma$; more on this below.
Summarise the adjoint pass:
\begin{algorithm}[Arnoldi's adjoint pass; paraphrased]
Assume $Q$, $H$, $c$, and $r$ as well as the gradients of $f$ with respect to those quantities.
Then, compute $\psi$ via computing $\gamma$ using $Z_r = 0$.
Then, for $k=K, ..., 1$, alternate the following two steps:
\begin{enumerate}
	\item Derive the next row of $\Psi$ by combining $Z_Q = 0$ with the projection constraint $Z_H = 0$
	\item Backward-substitute (``solve'') for the next row of $\Lambda$ (recall: we loop backwards)
\end{enumerate}
Finally, use $Z_c = 0$ to get the first row of $\Psi$ and solve for $\lambda$.
Then, return $\nabla_A f = \Lambda Q^\top$ and $\nabla_v f = \lambda c^\top$.
\end{algorithm}
The structure of the adjoint pass is similar to the forward pass (\Cref{table-forward-versus-adjoint-pass}).
\begin{table}[t]
\caption{Forward versus adjoint (backward) pass}
\label{table-forward-versus-adjoint-pass}
\begin{center}
\footnotesize
\begin{tabular}{ r  l l  l l }
\toprule
 & System matrix & Solve via & Recursively define & Using \\
\midrule
Forward & Lower triangular & Forward substitution & System matrix & Orthogonality\\
Adjoint & Upper triangular & Backward substitution & Right-hand side & Projection: $Z_H = 0$\\
\bottomrule
\end{tabular}
\end{center}
\end{table}
In the following, we will elaborate on each of those steps.
We assume that the reader knows how to solve a lower triangular linear system. 
We focus on constructing $\psi$ and $\Psi$ via $\Gamma$ and $\gamma$.

\subsubsection{Initialisation}
Initialisation of the adjoint pass implies computing $\psi$.
To get $\psi$, we need $\gamma$:
Consider multiplying $Z_r$ with $Q^\top$,
\begin{align}
0
= Q^\top Z_r
&= Q^\top \nabla_r f -Q^\top \Lambda e_K + \gamma
\tag{use the definition of $Z_r$}
\\
&= Q^\top \nabla_r f - \left( \nabla_H f + I_\ll \circ \Sigma\right)e_K + \gamma
\tag{since $Z_H = 0$}
\\
&=
(Q^\top \nabla_r f - \nabla_H f e_K) + (I_\ll \circ \Sigma) e_K + \gamma
\tag{reorder}
\\
&=
(Q^\top \nabla_r f - \nabla_H f e_K) + \gamma
\tag{($I_\ll \circ \Sigma) e_K = 0$}
\end{align}
where we use that the last column in $I_\ll \circ \Sigma$ consists entirely of zeros.
All other quantities are known.
Therefore, $\gamma$ is isolated and we identify
\begin{align}
\gamma = \nabla_H f e_K - Q^\top \nabla_r f. 
\end{align}
Next, use this $\gamma$ to build $\psi$,
\begin{align}
\psi(\gamma) = \nabla_r f + Q \gamma
\end{align}
and the initialisation step is complete.

\subsubsection{Recursion}
With $\psi$ in place, we get the last column of $\Lambda$.
To get the next column of $\Lambda$, we need to derive the right-hand side $\Psi$.
To get $\Psi$, we need $\Gamma$.

Multiply $Q^\top Z_Q$ to obtain
\begin{align}
0 
&= Q^\top Z_Q 
\\
&= 
Q^\top \nabla_Q f + Q^\top A^\top \Lambda - Q^\top \Lambda H^\top + Q^\top \lambda (e_1)^\top + I_\leq \circ \Gamma + (I_\leq \circ \Gamma)^\top
\tag{since $Q^\top Q = I$ and $Q^\top r = 0$}
\\
&=
Q^\top \nabla_Q f 
+ Q^\top A^\top \Lambda 
- (\nabla_H f + I_\ll \circ \Sigma) H^\top 
+ Q^\top \lambda (e_1)^\top 
+ I_\leq \circ \Gamma 
+ (I_\leq \circ \Gamma)^\top
\tag{since $Z_H = 0$}
\\
&=
(Q^\top \nabla_Q f - \nabla_H f H^\top)
+ Q^\top A^\top \Lambda
+ Q^\top \lambda (e_1)^\top
+ I_\leq \circ \Gamma
+ (I_\leq \circ \Gamma)^\top
+ (I_\ll \circ \Sigma)H^\top
\tag{reorder the terms}.
\end{align}
Since $H$ is Hessenberg, $(I_\ll \circ \Sigma) H^\top$ is strictly lower triangular.
Therefore, multiplication with $I_\geq$ removes $\Sigma$ from the expression,
\begin{subequations}
\begin{align}
0
&=
I_\geq \circ \left[Q^\top \nabla_Q f - \nabla_H f  H^\top + Q^\top A^\top \Lambda + Q^\top \lambda (e_1)^\top + I_\leq \circ \Gamma + (I_\leq \circ \Gamma)^\top \right]
\\
&=
I_\geq \circ \left[Q^\top \nabla_Q f - \nabla_H f  H^\top  + Q^\top A^\top \Lambda\right] + I_\geq \circ [Q^\top \lambda (e_1)^\top] + I_= \circ \Gamma + I_\geq \circ \Gamma^\top.
\end{align}
\end{subequations}
The term involving $\lambda$ can be simplified as follows:
Due to the presence of $(e_1)^\top$, we know that $Q^\top \lambda (e_1)^\top$ is lower triangular (in fact, it has a single nonzero column).
Thus, $I_\geq \circ (Q^\top \lambda (e_1)^\top)$ is proportional to $e_1 (e_1)^\top$,
\begin{align}
I_\geq \circ (Q^\top \lambda (e_1)^\top) 
&= 
[(Q e_1)^\top \lambda] \, e_1 (e_1)^\top
\\
&= 
c v^\top \lambda \, e_1 (e_1)^\top
\tag{since $Q e_1 = c v$}
\\
&=
c \,\nabla_c f \, e_1 (e_1)^\top
\tag{since $Z_c = 0$}
\end{align}
and all quantities are known; hence,
\begin{align}
I_= \circ \Gamma + I_\geq \circ \Gamma^\top
=
-I_\geq \circ [Q^\top \nabla_Q f - \nabla_H f H^\top + Q^\top A^\top \Lambda]
- c \,\nabla_c f \, e_1 (e_1)^\top
\end{align}
must hold.

Now, the most important observation is the following:
the last column of $I_= \circ \Gamma + I_\geq \circ \Gamma^\top$ depends on the last column of $Q^\top A^\top \Lambda$ and known quantities;
the penultimate column depends on the penultimate column of $\Gamma$, and so on.
But at the time of assembling the last column of $\Gamma$, the last column of $\Lambda$ is known!
More generally, we always know one more column of $\Lambda$ than of $\Gamma$, so we can recursively assemble $I_= \circ \Gamma + I_\geq \circ \Gamma^\top$:

Let 
\begin{align}
\symmetric{M} \coloneqq I_\geq \circ M + (I_> \circ M)^\top
\end{align}
be a symmetrisation operator.
We define it for the sole purpose of reconstructing
\begin{align}
\symmetric{I_= \circ \Gamma + I_\geq \circ \Gamma^\top} = I_\leq \circ \Gamma + (I_\leq \circ  \Gamma)^\top.
\end{align} 
Let us use it:
\begin{subequations}
\begin{align}
I_\leq \circ \Gamma + (I_\leq \circ  \Gamma)^\top
&=
\symmetric{
-I_\geq \circ [Q^\top \nabla_Q f - \nabla_H f H^\top + Q^\top A^\top \Lambda]
- c \,\nabla_c f \, e_1 (e_1)^\top
}
\\
&=
\symmetric{
-I_\geq \circ [Q^\top \nabla_Q f - \nabla_H f H^\top + Q^\top A^\top \Lambda]
}
- c \,\nabla_c f \, e_1 (e_1)^\top
.
\end{align}
\end{subequations}
This yields the next row/column of $\Gamma + \Gamma^\top$, and therefore the next row of $\Psi$.
From there, we can assemble the next column of $\Lambda$ and iterate.
\cref{figure-algorithm-pseudocode} (respectively \cref{algorithm-forward-pass,algorithm-backward-pass}) compare pseudocode for forward and adjoint passes.
\begin{figure}[t]
\fcolorbox{white}{brown!11!white}{%
	\begin{minipage}[c]{0.45\linewidth}
	\begin{algorithm}[Forward pass]
	\label{algorithm-forward-pass}
	Initialise $k=1$, $Q e_1$. Then, for $k=1, ..., K$:
	\begin{enumerate}
		\item Use orthogonality for a new row in the system matrix in \cref{equation-arnoldi-as-linear-system}.
		\item Solve for the next column of $Q$
		\item Optional: re-enforce $Q^\top Q = I$.
	\end{enumerate}
	Solve for $r$ and return $Q, H, r, c$.
	\end{algorithm}
	\end{minipage}
}
\hfill
\fcolorbox{white}{blue!7!white}{%
	\begin{minipage}[c]{0.45\linewidth}
	\begin{algorithm}[Backward pass]
	\label{algorithm-backward-pass}
	Initialise $k=K$, $\Lambda e_1$. Then, for $k=K, ..., 1$:
	\begin{enumerate}
		\item Use projection for a new column of the right-hand side $\Psi$
		\item Solve for the next column of $\Lambda$
		\item Optional: re-enforce $Z_H = 0$.
	\end{enumerate}
	Solve for $\lambda$ and return $\nabla_\theta \rho$ and $\nabla_v \rho$.
	\end{algorithm}
	\end{minipage}
}
\caption{Forward and backward pass of the Arnoldi iteration (paraphrased)}
\label{figure-algorithm-pseudocode}
\end{figure}
Altogether, the implementation of the adjoint pass is very similar to that of the forward pass.

At the final step, we obtain not the last column of $\Lambda$ but $\lambda$, though this is a byproduct of solving the triangular linear system.
It does not need further explanation.

\begin{remark}[$\Sigma$]
Like for gradients of QR decompositions \citep{roberts2020qr,walter2010algorithmic}, we never solve for $\Sigma$.
\end{remark}

\section{Setup for \Cref{table-accuracy-loss-hilbert}}
\label{appendix-section-accuracy-loss-hilbert}

\begin{wrapfigure}{r}{0.4\linewidth}
\includegraphics{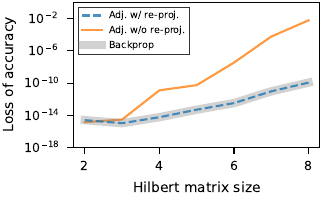}
\caption{Accuracy loss $\varepsilon$ when differentiating $\Ical$ for a Hilbert matrix of increasing size $N$. Uses double precision.}
\label{figure-accuracy-loss-all}
\end{wrapfigure}
To create \Cref{table-accuracy-loss-hilbert}, we implement an operator
\begin{align}
\Ical: A \mapsto (H, Q, r, c) \mapsto Q H Q^\top
\end{align}
where $(H, Q, r, c)$ are the result of a full-rank Arnoldi iteration (i.e. $K=N$).
For $K=N$, $Q H Q^\top=A$ and $\Ical$ must have an identity Jacobian; thus,
\begin{align}
\varepsilon \coloneqq \|I_{N^2} - \Ical\|_\text{RMSE}
\end{align}
measures the loss of accuracy when differentiating the Arnoldi iteration.
A small $\varepsilon$ is desirable.

Then, using double-precision, we construct a Hilbert matrix
\mbox{$A = \left[{1} / ({i + j + 1})\right]_{i,j=1}^N \in \Rbb^{N \times N}$}
which is a famously ill-conditioned matrix and a common test-bed for the loss of orthogonality in methods like the Lanczos and Arnoldi iteration \citep[e.g.][Table 7.1]{borm2012numerical}.
We evaluate three algorithms, all of which rely on the Arnoldi iteration with full reorthogonalisation on the forward-pass:
One algorithm does not re-project on the adjoint constraints, another one does, and for reference we compute $\varepsilon$ when ``backpropagating through'' the re-orthogonalised Arnoldi iteration as a third option.
\Cref{figure-wall-times-per-vjp} has demonstrated that the first two options beat the third one in terms of speed, but we consider numerical accuracy here.

We evaluate $\varepsilon$ for $N=1, ..., 8$ (see \Cref{figure-accuracy-loss-all}), and show the values for $N=8$ in \Cref{table-accuracy-loss-hilbert}.
The numerical accuracy of the re-projected adjoint method matches that of differentiating ``through'' re-orthogonalisation, and outperforms not re-projecting by a margin.

\section{Memory-efficient kernel-matrix-vector products in JAX}
\label{appendix-section-low-memory-matrix-vector-products}

Matrix-free linear algebra requires efficient matrix-vector products.
For kernel function $k=k(x,x')$, and input data $x_1, ..., x_N$, Gaussian process covariance matrices are of the form $A =[k(x_i, x_j)]_{i,j=1}^N$.
Matrix-vector products with $A$ thus look like
\begin{align}
v \mapsto A v= \left[ \sum_{j=1}^N k(x_i, x_j) v_j\right]_{i=1}^N
\end{align}
and can be assembled row-wise, either sequentially or parallely.

The more rows we assemble in parallel, the faster the runtime but also the higher the memory requirements, so we follow \citet{gardner2018gpytorch} and choose the largest number of rows of $A$ that still fit into memory, say $r$ such rows, and assemble $Av$ in blocks of $r$.
In practice, we implement this in JAX by combining $\texttt{jax.lax.map}$ and $\texttt{jax.vmap}$, but care has to be taken with reverse-mode automatic differentiation through $(v, \theta) \mapsto A(\theta) v$ because by default, reverse-mode differentiation stores all intermediate results.
To solve this problem, we place checkpoints around each such batch of rows, which reduces the memory requirements but roughly doubles the runtime.
(We place another checkpoint around each stochastic trace-estimation sample, which roughly doubles the runtime again.)

An alternative to doing this manually is the KeOps library \citep{charlier2021kernel}, which GPyTorch \citep{gardner2018gpytorch} builds on.
However, there currently exists no JAX-compatible interface to KeOps which is why we have to implement the above solution.

\Cref{figure-us-versus-keops} compares the runtime of our approach to that of KeOps custom CUDA code.
\begin{figure}[t]
\centering
\includegraphics[width=0.8\linewidth]{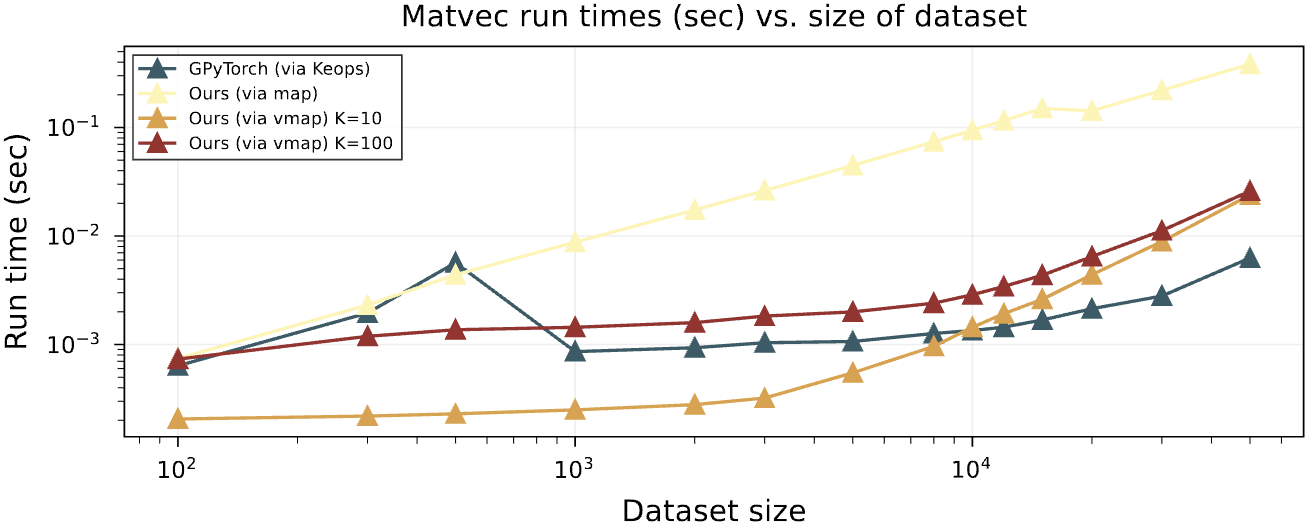}
\caption{For matrices with at least 10,000 rows/columns, KeOps remains the state of the art. This experiment uses a square-exponential kernel, on an artificial dataset with $d=3$ dimensions.}
\label{figure-us-versus-keops}
\end{figure}
We see that we are competitive, but roughly 5$\times$ slower for medium to large datasets.
Multiplying this with the 4$\times$ increase due to the checkpoints discussed above explains the 20$\times$ increase in runtime compared to GPyTorch.
Being 20$\times$ slower than GPyTorch per epoch is only due to the matrix-vector products, and has nothing to do with the algorithm contribution.
Future work should explore closing this gap with a KeOps-to-JAX interface.

\section{Experiment configurations for the Gaussian process study}
\label{appendix-experiment-details-gaussian-processes}

\paragraph{Data}
For the experiments we use the ``Protein'', ``KEGG (undirected'', ``KEGG (directed)'', ``Elevators'', and ``Kin40k'' datasets (\Cref{table-dataset-sources}, adapted from \citet{bartels2023adaptive}).
\begin{table}[ht]
\caption{Datasets used in this study.}
\label{table-dataset-sources}
\begin{center}
\begin{tabular}{llll}
\toprule
Dataset &  Source \\
\midrule
Protein &  Available here.\tablefootnote{%
\scriptsize Link: \texttt{http://archive.ics.uci.edu/dataset/265/physicochemical+properties+of+protein+tertiary+structure}} \\
Elevators & \citet{camachol1998inducing} \\
Kin40K &  \citet{schwaighofer2002transductive} \\
KEGG (undir) &  \citet{shannon2003cytoscape}  \\
KEGG (dir) &  \citet{shannon2003cytoscape} \\
\bottomrule
\end{tabular}
\end{center}
\end{table}
All are part of the UCI data repository, and accessible through there.

The data is subsampled to admit the train/test split of 80/20\%, and to admit an even division into the number of row partitions. 
More specifically, we use $10$ partitions for the kernel-matrix vector products.
This way, we have to discard less than $1\%$ of the data; e.g., on KEGG (undir), we use 63,600 instead of the original 63,608 points.

We calibrate a Mat\'ern prior with smoothness $\nu=1.5$, using 10 matrix-vector products per Lanczos iteration, conjugate gradients tolerance of $\epsilon = 1$, a rank-15 pivoted Cholesky preconditioner, and $10$ Rademacher samples.
We evaluate all samples sequentially (rematerialising on the backward pass to save memory, as discussed in \Cref{appendix-section-low-memory-matrix-vector-products}).
The conjugate-gradients tolerances are taken to be absolute (instead of relative), and the parametrisations of the Gaussian process models and loss functions match that of GPyTorch.

For every model, we calibrate an independent lengthscale for each input dimension, as well as an scalar observation noise, scalar output-scale, and the value of a constant prior mean.
All parameters are initialised randomly.
We use the Adam optimiser with learning rate 0.05 for 75 epochs.
All experiments are repeated for three different seeds.

\section{Partial differential equation data}
\label{appendix-section-pde-data}

We generate data for the differential equations as follows:
Recall the problem setup of a partial differential equation
\begin{align}
\frac{\partial^2}{\partial^2 t} u(t; x_1, x_2) = \omega(x_1, x_2)^2 \left( \frac{\partial^2}{\partial x_1^2} u(t; x_1, x_2) + \frac{\partial^2}{\partial x_2^2}  u(t; x_1, x_2)\right)
\end{align}
with Neumann boundary conditions. The coefficient field $\omega$ is space- but not time-dependent.

First, we discretise the Laplacian operator with central differences on an equidistant, tensor-product mesh that consists of $128$ points per dimension, which yields $128^2$ grid points.
The resulting second-order ordinary differential equation
\begin{align}
\frac{\diff^2}{\diff t^2} w = \omega^2 M w,
\end{align}
where $M$ is the discretised Laplacian,
is then transformed into a first-order differential equation 
\begin{align}
\frac{\diff}{\diff t} \begin{pmatrix} w \\\dot w \end{pmatrix} = \begin{pmatrix} 0 & I \\ \omega^2 M & 0\end{pmatrix} w \eqqcolon A w.
\end{align}
This equation is solved by the matrix exponential, and the system matrix $A$ is asymmtric (by construction), and highly sparse because $M$ is.
Matrix-vector products with $A$ are cheap, because we can implement them with \texttt{jax.scipy.signal.convolve2d}.

Then, we sample a true $\omega$ from a Gaussian process with a square-exponential covariance kernel, using lengthscale $\texttt{softplus}(-0.75)$ and output-scale $\texttt{softplus}(-10)$.
We sample from this process with the Lanczos algorithm \citep{pleiss2018constant} using Krylov-depth $K=32$.

Then, we use another Gaussian process with the same kernel, but lengthscale $\texttt{softplus}(0)$ and output scale $\texttt{softplus}(0)$, to sample $256$ initial distributions -- again with the Lanczos algorithm \citep{pleiss2018constant}.
These $256$ initial conditions are solved with Diffrax's implementation of Dopri8 \citep{prince1981high} using $128$ timesteps.
Some example input/output pairs are in \Cref{figure-pde-data}.
\begin{figure}[t]
\includegraphics[width=\linewidth]{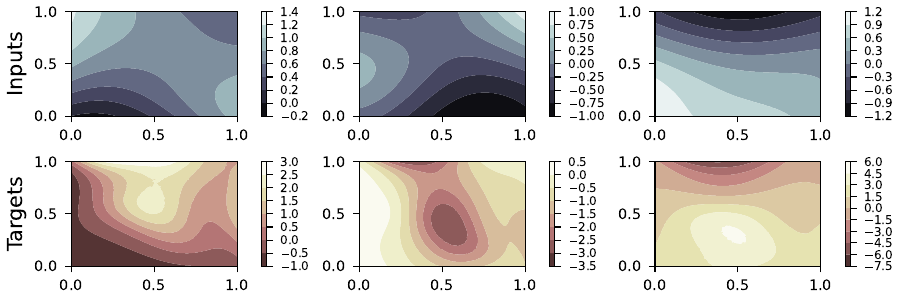}
\caption{Three exemplary input/output pairs from the PDE dataset.}
\label{figure-pde-data}
\end{figure}

This setup is similar to that of the WaveBench dataset \citep{liu2023wavebench}, with the main difference being that the WaveBench dataset uses a slightly different formulation of the wave equation.\footnote{To ensure radiating boundary conditions, \citet{liu2023wavebench} follow \citet{stanziola2023j}'s model of simulating the wave equations as a sytem of first-order equations.} 
We use the one above because it lends itself more naturally to matrix exponentials, which are at the heart of this experiment.

\section{Implementation details for the Bayesian neural network study}
\label{appendix-laplace}

\subsection{Bayesian neural networks with Laplace approximations}
Another possible application of the gradients of matrix functions is marginal-likelihood-based optimisation of Bayesian Neural Networks. 
Suppose $g_\theta(x)$ is the output of a neural network with parameters $\theta \in \mathbb{R}^P$. 
The choice of the model shall be denoted by $\mathcal{M}$ and consist of both continuous and discrete hyperparameters (such as network architecture, likelihood precision, prior precision, etc.). 
For some choice of prior given by 
\begin{align}
p(\theta \mid \mathcal{M})
\end{align}
and likelihood 
\begin{align}
p(y|x, \theta, , \mathcal{M}) = p(y  \mid  x, g_\theta(x), \mathcal{M})
\end{align}
we can specify a Bayesian model. The posterior distribution is then given by: 
\begin{align}
p(\theta, y  \mid  x, \mathcal{M}) \propto p(y  \mid  x, g_\theta(x), \mathcal{M}) p(\theta \mid  \mathcal{M}) d\theta.
\end{align}
The marginal likelihood is given by normalizing constant of this posterior, i.e. 
\begin{align}
p(y \mid x, \mathcal{M}) = \int p(y  \mid  x, g_\theta(x), \mathcal{M}) p(\theta \mid  \mathcal{M}) d\theta.
\end{align}
As suggested by \citet{mackay1992bayesian}, this marginal likelihood can be used for model selection in Bayesian neural networks. 
\citet{immer2021scalable} use the Laplace approximation of the posterior to obtain access to the marginal likelihood of the Bayesian neural network and its stochastic gradients. 

The Laplace approximation of the marginal likelihood is given by:
\begin{equation}
\log p(y \mid x, \mathcal{M}) \approx \log p(y, \theta_\text{MAP}  \mid x, \mathcal{M}) - \frac{1}{2} \log  \det\left( \frac{1}{2\pi} \bm{H}_{\theta_\text{MAP} } \right) 
\end{equation}

where $\bm{H}_{\theta_\text{MAP} } = - \nabla^2_\theta \log p(y, \theta_\text{MAP}  \mid x, \mathcal{M})$. 
Usual choices of the prior are $N(0, \alpha^{-1} \mathbb{I})$. Usually this Hessian is approximated with the \emph{generalized Gauss-Newton} (GGN) matrix \citep{immer2021improving} 
\begin{align}
\bm{H}_{\theta_\text{MAP}} 
\approx A(\alpha) 
\coloneqq 
\sum_{j=1}^J 
[D_\theta g_{\theta_\text{MAP} }])(x_j)^\top 
[D_g^2 \rho](y_j, g_{\theta_\text{MAP}}(x_j)) 
[D_\theta g_{\theta_\text{MAP} }])(x_j)^\top
+ \alpha^2 I
\end{align}
where $D^2\rho$ is the Hessian of the loss, and $D_\theta g$ the Jacobian of $g$ (recall \Cref{equation-gauss-newton-matrix}).
This objective is used to optimize the prior precision of the model or any continuous model hyperparameters. 
Matrix-vector products with the GGN matrix can be accessed through automatic differentiation using Jacobian-vector and vector-Jacobian products.
With these efficient matrix-vector products, one can estimate the log-determinant of GGN using matrix-free techniques like the Lanczos iteration.

To make predictions using the Laplace approximation of the posterior, we also need to sample from the normal distribution ${N}(\theta_\text{MAP}, A^{-1})$. Samples from this distribution can be written as: 
\begin{align}
\theta = \theta_\text{MAP}  + A^{-1/2} \epsilon
\end{align}
where $\epsilon \sim N(0, I)$. The main bottleneck in this computation is the inversion and matrix square root of the GGN matrix, and we implement it with a Lanczos iteration using $f(x) = x^{-1/2}$.
Since the GGN is empirically known to have low-rank \citep{papyan2020traces}, doing a few Lanczos iterations can get us close to an accurate estimation.

\subsection{Experiment setup}
We estimate the diagonal of the GGN stochastically via (``$\circ$'' is the element-wise product) \citep{bekas2007estimator} 
\begin{align}
\diagonal{A} = \Ebb[v \circ Av] \approx \frac{1}{L} \sum_{\ell=1}^L v_\ell \circ A v, \quad \Ebb[v v^\top] = I.
\end{align}
We use $150$ matrix-vector products for both diagonal calibration and our Lanczos-based estimation.
We use $30$ Monte-Carlo samples to estimate the log-likelihoods for evaluating the test metrics, and 
we use \texttt{places365} \citep{zhou2014learning} as an out-of-distribution dataset to compute OOD-AUROC. 
We also compute the expected calibration error (ECE) \citep{naeini2015obtaining} of the model. 

\paragraph{Data:}
We show scalability by doing Laplace approximation on Imagenet1k image classification \citep{deng2009imagenet}. The training set consists of approximately 1.2 million images, each belonging to one of 1000 classes. 
We find that we can take small subsets of this dataset and still converge to the same prior precision. 
Our computational budget allows us to use 10 percent of the samples for each class.
However, even for very small subsamples of the data, we converge to a very similar prior precision.

\paragraph{Method:}
To optimize the prior precision we use the marginal likelihood as the objective. We use the RMSprop optimizer with a learning rate of 0.01 for 100 epochs for optimizing both the diagonal and Lanczos approximations of the GGN.

\newpage
\section*{NeurIPS Paper Checklist}

\begin{enumerate}

\item {\bf Claims}
    \item[] Question: Do the main claims made in the abstract and introduction accurately reflect the paper's contributions and scope?
    \item[] Answer: \answerYes{} %
    \item[] Justification: The main contributions, adjoint systems for the Lanczos and Arnoldi iterations, are explained in \Cref{section-method}. The three case studies are in \Cref{section-case-study-pde,section-case-study-gaussian-processes,section-case-study-bayesian-neural-networks}.

\item {\bf Limitations}
    \item[] Question: Does the paper discuss the limitations of the work performed by the authors?
    \item[] Answer: \answerYes{} %
    \item[] Justification: The limitations are discussed in the paragraph titled ``Limitations and future work'' on page 4.
    All assumptions (i.e., \Cref{assumption-matvecs-only,assumption-small-matrix-function-differentiable}) are contextualised in the sentences before and after they are have been introduced.

\item {\bf Theory Assumptions and Proofs}
    \item[] Question: For each theoretical result, does the paper provide the full set of assumptions and a complete (and correct) proof?
    \item[] Answer: \answerYes{} %
    \item[] Justification: The main contributions, \Cref{theorem-adjoint-system-of-arnoldi,theorem-adjoint-system-of-lanczos,corollary-parameter-gradients-arnoldi}, are proven in \Cref{appendix-section-proof-of-corollary-parameter-gradients-arnoldi,appendix-section-proof-of-theorem-adjoint-system-of-lanczos,appendix-section-proof-of-theorem-adjoint-system-of-arnoldi}. \Cref{appendix-section-solving-the-adjoint-system-in-detail} discusses solving the adjoint system in full detail.

    \item {\bf Experimental Result Reproducibility}
    \item[] Question: Does the paper fully disclose all the information needed to reproduce the main experimental results of the paper to the extent that it affects the main claims and/or conclusions of the paper (regardless of whether the code and data are provided or not)?
    \item[] Answer: \answerYes{} %
    \item[] Justification: The most important information about the experiment setup is a part of the main paper in \Cref{section-case-study-pde,section-case-study-gaussian-processes,section-case-study-bayesian-neural-networks}; further information can be found in \Cref{appendix-experiment-details-gaussian-processes,appendix-section-pde-data,appendix-laplace,appendix-section-low-memory-matrix-vector-products}.
    \Cref{appendix-section-accuracy-loss-hilbert,appendix-section-mode-vjp-wall-times} discuss the setup for \Cref{figure-wall-times-per-vjp,table-accuracy-loss-hilbert}.
    Code will be published upon acceptance.

\item {\bf Open access to data and code}
    \item[] Question: Does the paper provide open access to the data and code, with sufficient instructions to faithfully reproduce the main experimental results, as described in supplemental material?
    \item[] Answer: \answerYes{} %
    \item[] Justification: Code has been submitted as a part of the supplementary material, and will be published upon acceptance. 

\item {\bf Experimental Setting/Details}
    \item[] Question: Does the paper specify all the training and test details (e.g., data splits, hyperparameters, how they were chosen, type of optimizer, etc.) necessary to understand the results?
    \item[] Answer: \answerYes{} %
    \item[] Justification: See the answer to ``4. Experimental Result Reproducibility'' above.

\item {\bf Experiment Statistical Significance}
    \item[] Question: Does the paper report error bars suitably and correctly defined or other appropriate information about the statistical significance of the experiments?
    \item[] Answer: \answerYes{} %
    \item[] Justification: All case studies report mean and standard deviations of multiple runs. The only exception is the Bayesian neural network example, which uses a single training run (but evaluates test metrics on multiple seeds).

\item {\bf Experiments Compute Resources}
    \item[] Question: For each experiment, does the paper provide sufficient information on the computer resources (type of compute workers, memory, time of execution) needed to reproduce the experiments?
    \item[] Answer: \answerYes{} %
    \item[] Justification: The introductory part of the appendix contains a paragraph titled ``Compute''.
    
\item {\bf Code Of Ethics}
    \item[] Question: Does the research conducted in the paper conform, in every respect, with the NeurIPS Code of Ethics \url{https://neurips.cc/public/EthicsGuidelines}?
    \item[] Answer: \answerYes{} %
    \item[] Justification: We use datasets that are either self-created (\Cref{section-case-study-pde,appendix-section-pde-data}), or common test-cases for machine learning methods (UCI datasets, ImageNet) or numerical algorithms (SuiteSparse matrix collection).

\item {\bf Broader Impacts}
    \item[] Question: Does the paper discuss both potential positive societal impacts and negative societal impacts of the work performed?
    \item[] Answer: \answerNo{} %
    \item[] Justification: This research provides a foundational algorithm for computational sciences, and societal impact is difficult if not impossible to predict. 
    
\item {\bf Safeguards}
    \item[] Question: Does the paper describe safeguards that have been put in place for responsible release of data or models that have a high risk for misuse (e.g., pretrained language models, image generators, or scraped datasets)?
    \item[] Answer: \answerNA{} %
    \item[] Justification: This paper does not contribute data or models that would require safeguards. 

\item {\bf Licenses for existing assets}
    \item[] Question: Are the creators or original owners of assets (e.g., code, data, models), used in the paper, properly credited and are the license and terms of use explicitly mentioned and properly respected?
    \item[] Answer: \answerYes{} %
    \item[] Justification: All assets used in the experiments have been cited. See also the answer to ``9. Code of Ethics''. 

\item {\bf New Assets}
    \item[] Question: Are new assets introduced in the paper well documented and is the documentation provided alongside the assets?
    \item[] Answer: \answerYes{} %
    \item[] Justification: The code (attached to this submission) is documented. The appendices contain all other information.

\item {\bf Crowdsourcing and Research with Human Subjects}
    \item[] Question: For crowdsourcing experiments and research with human subjects, does the paper include the full text of instructions given to participants and screenshots, if applicable, as well as details about compensation (if any)? 
    \item[] Answer: \answerNA{} %
    \item[] Justification: This study does not involve crowdsourcing nor human subjects.

\item {\bf Institutional Review Board (IRB) Approvals or Equivalent for Research with Human Subjects}
    \item[] Question: Does the paper describe potential risks incurred by study participants, whether such risks were disclosed to the subjects, and whether Institutional Review Board (IRB) approvals (or an equivalent approval/review based on the requirements of your country or institution) were obtained?
    \item[] Answer: \answerNA{} %
    \item[] Justification: This paper does not involve crowdsourcing nor research with human subjects.

\end{enumerate}

\end{document}